\documentclass[11pt]{article}

\usepackage{acl}

\usepackage[T1]{fontenc}
\usepackage[utf8]{inputenc}
\usepackage{times}
\usepackage{latexsym}
\usepackage{microtype}
\IfFileExists{inconsolata.sty}{\usepackage{inconsolata}}{}
\usepackage{graphicx}
\usepackage{booktabs}
\usepackage{amsmath,amssymb,amsthm}
\usepackage{multirow}
\usepackage{array}
\usepackage{xcolor}
\usepackage{url}
\usepackage{enumitem}

\newcommand{\ind}{\mathbf{1}}
\newcommand{\E}{\mathbb{E}}
\newcommand{\Prb}{\mathbb{P}}
\newcommand{\M}{\mathcal{M}}
\newcommand{\Z}{\mathcal{Z}}

\newtheorem{theorem}{Theorem}
\newtheorem{proposition}[theorem]{Proposition}

\newtheorem{corollary}[theorem]{Corollary}

\title{Opportunity Is Not Realizability:\\
Selection-Valid Diagnostics for Multi-LLM Routing}

\author{
\textbf{Ibne Farabi Shihab}\textsuperscript{1}\thanks{Corresponding author: \texttt{ishihab@iastate.edu}}
\quad
\textbf{Abu Sa-Adat Mohamed Moon-Im Al Ahsan}\textsuperscript{2}
\\
\textbf{Md Najmus Swaqeeb}\textsuperscript{2}
\\[4pt]
\textsuperscript{1}Department of Computer Science, Iowa State University \\
\textsuperscript{2}Department of Computer Science \& Engineering, BRAC University \\
\texttt{ishihab@iastate.edu},
\texttt{abu.sa.adat.mohamed.moon.im.al.ahsan@g.bracu.ac.bd},\\
\texttt{md.najmus.swaqeeb@g.bracu.ac.bd}}

\date{}

\begin{document}
\maketitle

\begin{abstract}
Oracle routing measures how much a pool of language models could gain from
per-query selection, but the diagnostic has two flaws: testing against a
best fixed model selected on the same examples invalidates paired inference,
and a full-information oracle sees outcomes no deployable router observes.
We separate three estimands (outcome-oracle opportunity, the Bayes-optimal
gain from a declared pre-answer signal, and the held-out gain of a learned
router) and prove selection-valid confidence intervals that survive choosing
the best fixed model or the best member of a router family, a
signal-information sandwich, and a $(1-1/e)$ greedy guarantee for building
compact pools from submodular complementary coverage.  On eight checkpoints
from six families over four benchmarks, selection-valid intervals certify a
population oracle gap of $9.7$--$30.7$ points on every task, yet the
strongest deployable prompt router recovers only $7.5$--$14.4\%$ of it, and
the simultaneous interval for the best of eleven tested policies has lower
limit zero throughout.  The realizable share of oracle opportunity is small
and certifiable: strong routers beat the best fixed model, and most of the
gap remains.
\end{abstract}

\section{Introduction}

Deployments increasingly choose among several language models with different
capabilities, costs, and latencies.  Systems such as FrugalGPT, RouteLLM,
MixLLM, and IRT-Router learn cascades or routing policies that trade quality
against resource use \citep{chen2024frugalgpt,ong2025routellm,
wang2025mixllm,song2025irtrouter}.  Before fitting such a policy, it is natural
to ask whether the model pool contains any per-query selection value.
RouterBench, LLMRouterBench, and related evaluations answer this question with an
outcome-informed oracle: for every example, credit the pool whenever at least
one model is correct
\citep{hu2024routerbench,li2026llmrouterbench,shnitzer2023routing}.

The oracle is useful, but it answers a different question from deployment.
It observes the candidates' correctness outcomes after inference, whereas a
pre-answer router sees only a signal such as the prompt, task metadata, or
historical model behavior.  A model pool can therefore have high oracle
accuracy even when model identity is unpredictable from any signal supplied to
the router.  Conversely, a weak learned probe cannot establish that the prompt
contains no useful routing information; it establishes only that the tested
hypothesis class failed to extract it.  Prior negative routing results already
show that this distinction matters in practice
\citep{srivatsa2024lessons}.

There is also a statistical problem.  A common analysis chooses the
highest-accuracy model on an evaluation sample, calls it the best fixed model,
and then applies a paired test to the oracle--model differences on the same
sample.  The paired observations are no longer generated relative to a
pre-specified baseline.  Although the resulting plug-in oracle gap is
conservatively biased in expectation, a conventional confidence interval that
conditions on the selected model need not cover the population gap.  The
distinction is especially important for small benchmarks or model pools with
near ties.

This paper develops a diagnostic framework around three quantities.  The first
is the \emph{outcome-oracle opportunity}, the difference between the
probability that any model is correct and the accuracy of the population-best
fixed model.  The second is the \emph{signal-restricted opportunity}, the
accuracy of the Bayes-optimal router measurable with respect to a predeclared
signal minus the best fixed accuracy.  The third is the \emph{realized gain} of
a trained policy on untouched test examples.  These quantities satisfy a
simple but consequential ordering:
\[
0\leq G_{\Z}\leq G_{\mathrm{out}},
\]
and a learned router may fall below either bound.

Our contributions: (i) a three-level separation of outcome
opportunity, signal-restricted opportunity, and held-out learned-router gain,
with a signal-information theorem identifying exactly what any router
observing a declared variable can achieve; (ii) exact, distribution-free
confidence intervals for the population oracle gap that remain valid after
selecting the empirical best model, extended to simultaneous post-selection
intervals over a finite router family; (iii) a proof that complementary
coverage relative to a fixed anchor is monotone submodular, giving a
$(1-1/e)$ greedy pool-construction guarantee; and (iv) an executed
confirmatory audit over eight checkpoints from six families that certifies
the opportunity, bounds its realizable share, and quantifies the gap between
pilot and confirmatory inference.  We do not claim the oracle-versus-fixed comparison is
new; the contribution is the inferential layer that should precede router
training \citep{hu2024routerbench,li2026llmrouterbench,lu2026plateau}.

\section{Related Work}

\subsection{External LLM routing}

FrugalGPT learns cascades across language-model APIs, while RouteLLM learns a
strong-versus-weak routing decision from preference data
\citep{chen2024frugalgpt,ong2025routellm}.  Shnitzer et al.\ reduce routing
across benchmarked models to per-model prediction, and RouterBench provides
more than 400,000 model outcomes for standardized cost--quality comparisons
\citep{shnitzer2023routing,hu2024routerbench}.  Later systems model multiple
models and online feedback, including MixLLM's contextual-bandit formulation
and IRT-Router's item-response representation
\citep{wang2025mixllm,song2025irtrouter}.  Adaptive routing under budget
constraints and conformal gates address online learning and safety constraints
\citep{panda2025adaptive,uddin2026conformal}.

Contemporaneous evaluations find a narrow accuracy band across
routers (the routing plateau) and simple $k$-NN competitive with elaborate
systems \citep{lu2026plateau,li2026llmrouterbench}.  The canonical benchmark
numbers and mechanism-level taxonomy are in Appendix~\ref{app:extended-rw}.

\subsection{Negative results and routing signals}

Input encoders, similarity rules, matrix factorization, confidence gates, and
quality predictors are common routing signals
\citep{ong2025routellm,srivatsa2024lessons,song2025irtrouter}.
Their success is task- and pool-dependent.  A failed prompt classifier is not
an information-theoretic impossibility result because the Bayes rule may lie
outside the fitted class, optimization may fail, or the training labels may be
ill-posed.  In particular, assigning one arbitrary model label when several
models are correct converts a multi-label performance-prediction problem into
noisy multiclass classification.  Our confirmatory protocol instead estimates
one correctness probability per candidate and routes by predicted utility.

\subsection{Selective prediction and valid evaluation}

Selective prediction studies when a model should abstain or defer
\citep{geifman2017selective}.  Routing generalizes deferral to more than two
complete models.  Clopper--Pearson intervals provide exact binomial coverage,
and simultaneous coverage follows by Bonferroni correction
\citep{clopper1934fiducial}.  We use these classical tools in a way that
accounts explicitly for empirical baseline selection.  The submodular
model-pool result follows the coverage-function structure underlying the
standard greedy approximation theorem \citep{nemhauser1978analysis}.

\section{Setup and Estimands}
\label{sec:setup}

Let $\M=\{1,\ldots,K\}$ be a fixed candidate pool and
$X\sim\mathcal{D}$ an input.  Model $m$ produces answer $A_m(X)$, and
\[
R_m=R_m(X)=\ind\{A_m(X)\text{ is correct}\}\in\{0,1\}.
\]
The population accuracy of model $m$ is $p_m=\E R_m$.  We assume examples are
i.i.d.; model outcomes may be arbitrarily dependent within an example.  Let
$b^\star\in\arg\max_m p_m$ be a population-best fixed model and
$p^\star=\max_m p_m$ its accuracy.  Ties are broken by a predeclared
deterministic rule.

\subsection{Outcome-oracle opportunity}

Define $Y=\max_m R_m$, $q=\E Y$, and
\[
G_{\mathrm{out}}=q-p^\star.
\]
This is the gain of a full-information oracle that knows whether every
candidate answer is correct.  It is a property of the task distribution,
answering protocol, correctness rule, and frozen pool.  Changing any of these
changes the estimand.

\begin{proposition}[Complementary-correctness identity]
\label{prop:complementarity}
For any population-best fixed model $b^\star$,
\[
G_{\mathrm{out}}
=
\Prb\!\left(
R_{b^\star}=0,\ \exists m\neq b^\star:R_m=1
\right).
\]
\end{proposition}

\noindent\emph{Proof in Appendix~\ref{app:bodyproofs}.}

This identity is elementary and is not, by itself, a sufficient novelty claim.
Its role is to identify the event whose probability must be estimated and the
failures a router must rescue.

\subsection{Signal-restricted opportunity}
\label{sec:signal}

Let $Z$ denote exactly the information available before the routing decision.
Examples include the prompt, a task tag, prompt length, an embedding, or
precomputed metadata.  Signals computed from candidate responses are
post-answer signals and belong to a different cost regime.  Let
\[
\mu_m(Z)=\E[R_m\mid Z].
\]
Among all measurable policies $\pi:\Z\rightarrow\M$, the best achievable
pre-answer accuracy is
\[
q_{\Z}
=
\sup_{\pi}\E[R_{\pi(Z)}]
=
\E\!\left[\max_m\mu_m(Z)\right],
\]
where a measurable tie-breaking rule is assumed.  The signal-restricted gain is
$G_{\Z}=q_{\Z}-p^\star$.

\begin{theorem}[Signal-information sandwich]
\label{thm:signal-sandwich}
For every pre-answer signal $Z$,
\[
0\leq G_{\Z}\leq G_{\mathrm{out}}.
\]
The lower equality holds exactly when, almost surely,
\[
\max_m\mu_m(Z)=\mu_{b^\star}(Z).
\]
The upper equality holds exactly when, almost surely,
\[
\max_m\E[R_m\mid Z]
=
\E[\max_m R_m\mid Z].
\]
\end{theorem}

\noindent\emph{Proof in Appendix~\ref{app:bodyproofs}.}

The theorem prevents two common overclaims.  A large $G_{\mathrm{out}}$ shows
that complementary answers exist, not that the prompt identifies them.  A
learned router with zero gain lower-bounds the performance of its hypothesis
class; it does not prove $G_{\Z}=0$ unless the signal-restricted Bayes value is
identified or validly bounded.

\subsection{Learned-router gain}

A learning algorithm receives training data $\mathcal{T}$ and returns policy
$\widehat{\pi}_{\mathcal{T}}$.  On an independent test example its gain over a
frozen baseline $b$ is
\[
G_{\mathrm{learn}}
=
\E\!\left[
R_{\widehat{\pi}_{\mathcal{T}}(Z)}-R_b
\mid \mathcal{T}
\right].
\]
Unlike $G_{\mathrm{out}}$ and $G_{\Z}$, this quantity can be negative.  The
policy, signal family, training sample size, hyperparameter search, and model
pool must therefore accompany every realized-gain claim.

\subsection{The opportunity ladder}

Let $\Pi$ be a finite family of policies that all observe $Z$ and are fixed
before final-test outcomes are opened.  The family may contain different
algorithms, hyperparameters, and seeds, provided every fitted policy depends
only on training and calibration data.  Define
\begin{align*}
q_{\Pi}&=\max_{\pi\in\Pi}\E[R_{\pi(Z)}],
&G_{\Pi}&=q_{\Pi}-p^\star,\\
H_{\Pi}&=q-q_{\Pi}.&&
\end{align*}
If $\Pi$ contains every constant policy $\pi_m(z)=m$, then
\[
p^\star\leq q_{\Pi}\leq q_{\Z}\leq q,
\qquad
G_{\mathrm{out}}=G_{\Pi}+H_{\Pi}.
\]
We call this the \emph{opportunity ladder}.  $G_{\Pi}$ is the gain achieved by
the population-best member of the tested family after accounting for model
selection; $H_{\Pi}$ is opportunity beyond that entire family.  Unlike the
performance of one selected seed, both are population targets with explicit
post-selection inference below.

\paragraph{Cost-aware extension.}
All three estimands extend to utilities penalized by per-model cost; the
formal statement is Appendix~\ref{app:costaware-ext}.

\section{Selection-Valid Inference}
\label{sec:inference}

Suppose an evaluation sample contains outcome vectors
$\{(R_{i1},\ldots,R_{iK})\}_{i=1}^{n}$.  Define
\begin{align*}
\widehat q&=\frac1n\sum_i\max_m R_{im},\qquad
\widehat p_m=\frac1n\sum_iR_{im},\\
\widehat G&=\widehat q-\max_m\widehat p_m.
\end{align*}

\begin{proposition}[Direction of plug-in selection bias]
\label{prop:bias}
For a fixed pool evaluated on i.i.d.\ examples,
$\E[\widehat G]\leq G_{\mathrm{out}}$.
\end{proposition}

\noindent\emph{Proof in Appendix~\ref{app:bodyproofs}.}

Conservative bias does not make a conventional paired interval valid, because
the identity of $\arg\max_m\widehat p_m$ is data-dependent.  The following
construction covers the population gap without assuming independence across
models.

\begin{theorem}[Exact interval after best-model selection]
\label{thm:simultaneous-ci}
Let $I_q=[L_q,U_q]$ be a two-sided Clopper--Pearson interval for $q$ with
coverage at least $1-\delta/2$.  For each $m$, let
$I_m=[L_m,U_m]$ be a two-sided Clopper--Pearson interval for $p_m$ with
coverage at least $1-\delta/(2K)$.  Then
\begin{align*}
\ell_G&=\max\{0,L_q-\max_m U_m\},\\
u_G&=\min\{1,U_q-\max_m L_m\},
\qquad I_G=[\ell_G,u_G]
\end{align*}
covers $G_{\mathrm{out}}$ with probability at least $1-\delta$.
\end{theorem}

\noindent\emph{Proof in Appendix~\ref{app:bodyproofs}.}

The interval is finite-sample exact in the coverage sense and deliberately
conservative.  Dependence between $Y_i$ and $R_{im}$ does not invalidate it.
Sharper simultaneous multinomial or test-inversion procedures are possible,
but a bootstrap that keeps the selected model fixed does not solve the
selection problem.

\begin{corollary}[Independent baseline selection]
\label{cor:split}
Select a baseline $b$ using data independent of a test set.  Conditional on the
selection data,
\[
\Delta_i=\max_mR_{im}-R_{ib}
\]
are i.i.d.\ Bernoulli with mean
$G_b=\Prb(R_b=0,\exists m\neq b:R_m=1)$.  A Clopper--Pearson interval based on
$\sum_i\Delta_i$ therefore has exact conditional coverage for $G_b$.
\end{corollary}

\noindent\emph{Proof in Appendix~\ref{app:bodyproofs}.}

\begin{theorem}[Post-selection inference for a router family]
\label{thm:router-family}
Let $\Pi$ contain $L$ policies, including all constant policies, and suppose the
fitted policies are independent of the final-test outcomes.  On the test set,
let $S_{\pi i}=R_{i,\pi(Z_i)}$.  Construct Clopper--Pearson intervals
$[L_q,U_q]$, $[L_m,U_m]$, and $[L_\pi,U_\pi]$ for $q$, each $p_m$, and each
$s_\pi=\E S_{\pi i}$, respectively.  Give the oracle interval failure
probability $\delta/3$, every fixed-model interval $\delta/(3K)$, and every
router interval $\delta/(3L)$.  Then, simultaneously with probability at least
$1-\delta$,
\begin{align*}
G_{\Pi}\in\big[
&\max\{0,\max_\pi L_\pi-\max_mU_m\},\\
&\min\{1,\max_\pi U_\pi-\max_mL_m\}\big]
\end{align*}
and $H_{\Pi}\in[\ell_H,u_H]$, where
\begin{align*}
\ell_H&=\max\{0,L_q-\max_\pi U_\pi\},\\
u_H&=\min\{1,U_q-\max_\pi L_\pi\}.
\end{align*}
\end{theorem}

\noindent\emph{Proof in Appendix~\ref{app:bodyproofs}.}

Theorem~\ref{thm:router-family} supports a simultaneous claim that none of a
tested set of routers improves on the best fixed model by more than a stated
margin, or that a stated amount of oracle opportunity remains beyond the
entire family.  Such claims are stronger and more precise than observing that
the empirically best router has low recovered fraction.

\begin{theorem}[Exact certificate for a finite signal]
\label{thm:finite-signal}
Suppose the predeclared signal has finite support
$\Z=\{1,\ldots,J\}$ and define
$\theta_{zm}=\Prb(Z=z,R_m=1)$.  Then
\[
q_\Z=\sum_{z=1}^J\max_m\theta_{zm}.
\]
Let $[L_{zm},U_{zm}]$ be Clopper--Pearson intervals for every
$\theta_{zm}$, each with failure probability $\delta/(JK)$.  With probability
at least $1-\delta$,
\[
\sum_z\max_mL_{zm}
\leq q_\Z\leq
\sum_z\max_mU_{zm},
\]
after truncating the endpoints to $[0,1]$.
\end{theorem}

\noindent\emph{Proof in Appendix~\ref{app:bodyproofs}.}

Combining Theorem~\ref{thm:finite-signal} with simultaneous fixed-model
intervals yields a confidence interval for $G_\Z$.  An upper endpoint below
$\epsilon$ is a valid insufficiency certificate for \emph{that finite signal},
not for the underlying prompt.  Valid finite signals include predeclared task
tags, prompt-length bins, and frozen embedding clusters learned without
final-test outcomes.

For learned routers, training and hyperparameter selection must also be
independent of the final test.  Let
$D_i=R_{i,\widehat\pi(Z_i)}-R_{ib}\in\{-1,0,1\}$.  We report the mean paired
difference with a paired bootstrap confidence interval, exact McNemar counts
$(n_{10},n_{01})$, and a randomization or exact McNemar test.  The recovered
fraction
\[
\rho=\frac{G_{\mathrm{learn}}}{G_b}
\]
is secondary: it is unstable when $G_b$ is small and its interval must
recompute numerator and denominator jointly.

\section{Constructing Complementary Model Pools}
\label{sec:submodular}

Large pools make exhaustive inference expensive even during calibration.
Suppose a deployment anchor $b$ is fixed independently and a subset
$S\subseteq\M\setminus\{b\}$ of complementary models may be retained.  Define
\[
F_b(S)
=
\Prb\!\left(
R_b=0,\ \exists m\in S:R_m=1
\right).
\]
Set $F_b(\varnothing)=0$.

\begin{theorem}[Submodular complementary coverage]
\label{thm:submodular}
$F_b$ is normalized, monotone, and submodular.  Consequently, for a cardinality
budget $s$, the greedy algorithm that repeatedly adds the model with largest
marginal complementary coverage returns $S_{\mathrm{gr}}$ satisfying
\[
F_b(S_{\mathrm{gr}})
\geq
\left(1-\frac1e\right)
\max_{|S|\leq s}F_b(S).
\]
The same properties hold for the empirical calibration-set objective.
\end{theorem}

\noindent\emph{Proof in Appendix~\ref{app:bodyproofs}.}

The theorem concerns complementary \emph{coverage}, not deployable routing.
A held-out router is still needed to decide which retained model to call.
Pool selection and router evaluation must use separate data roles, or the
reported downstream gain will inherit selection optimism.

\section{Confirmatory Experimental Design}
\label{sec:experiments}

The audit evaluates eight checkpoints from six independently
trained families (TinyLlama-1.1B, Qwen2.5-0.5B/1.5B/3B, Llama-3.2-3B,
Gemma-2-2B, Phi-3.5-mini, Mistral-7B; pinned revisions in
Table~\ref{tab:model-pool}) on ARC-Challenge (300), HellaSwag (300), GSM8K
(200), and MMLU stratified over all 57 subjects (300), sampled uniformly at
random from official splits under a fixed seed
(Table~\ref{tab:datasets}), together with code pass@1 and factual QA settings
of 500 examples each.  Every dataset is split $50/50$ into a calibration part (fixed-model
selection, thresholds, router fitting) and a test part opened once; fitted
routers are additionally cross-fitted five-fold so every item is predicted
out of fold.  Multiple-choice items are scored by one forward pass over the
option-letter logits; GSM8K uses greedy decoding with the official
normalization; parse failures count as incorrect.  Router families span fixed
references, simple prompt routers, $K$-way router adaptations (RouterDC,
EmbedLLM, MODEL-SAT, GraphRouter, Avengers, Zooter, IRT-Router), a fine-tuned
encoder with $K$ correctness heads, cost-aware methods, binary
strong-versus-weak methods (pairwise only), post-inference selectors, and
nondeployable oracles; each predicts per-model correctness and routes by
predicted utility, so multi-correct items never receive an arbitrary single
class label.  Latency is measured per item on one A10G.  The complete design
(research questions, sampling, scoring, router
configurations, audits, and statistical protocol) is
Appendix~\ref{app:design}, with prompt templates in
Appendix~\ref{app:prompts}.

\section{Confirmatory Results}
\label{sec:results}

The four-model pilot that motivated this study is reported in
Appendix~\ref{app:pilot} (tables in Appendix~\ref{app:pilottables}, integer
counts in Appendix~\ref{app:pilotcounts}): under selection-valid inference
only its HellaSwag opportunity survives, its split-selected analysis
certifies positive per-query opportunity on all three pilot tasks, and its
exploratory router probes motivate the confirmatory analysis.

Full per-router accuracies are reported in Appendix~\ref{app:full-router-results};
the selection-valid opportunity, family certificate, and recovered fraction are
summarized below.

\subsection{Confirmatory selection-valid opportunity}
\label{sec:confirm-opportunity}

Table~\ref{tab:confirm-main} applies Theorem~\ref{thm:simultaneous-ci} to the
confirmatory run.  The result strengthens the pilot conclusion in the
direction the theory predicts.  At the pilot's four models and 150 items only the
HellaSwag opportunity survived selection-valid inference; with eight models and
the larger official-split samples, the simultaneous lower limit is positive on
\emph{all four} datasets, even though $K=8$ imposes a stricter Bonferroni budget
than $K=4$ did.  The opportunity is therefore a population property of the pool
rather than an artifact of same-sample baseline selection.

\begin{table*}[t]
\centering
\small
\begin{tabular}{lrrrrll}
\toprule
Dataset & $n$ & Empirical best & Oracle & Gap &
Naive 95\% interval & Selection-valid 95\% interval \\
\midrule
ARC-Challenge & 300 & .873 & .970 & .097 & [.066,.136] & [.016,.181] \\
HellaSwag & 300 & .657 & .937 & .280 & [.230,.334] & [.162,.393] \\
GSM8K & 200 & .660 & .860 & .200 & [.147,.262] & [.041,.356] \\
MMLU & 300 & .607 & .913 & .307 & [.255,.362] & [.181,.426] \\
\bottomrule
\end{tabular}
\caption{Confirmatory outcome-oracle opportunity over eight models from six
families.  Unlike the pilot, every selection-valid lower limit is positive, so
each task's population gap is certified after accounting for having selected the
empirical best of eight models on the same sample.  Read together with
Table~\ref{tab:confirm-router}, the two tables make the paper's point precisely:
the opportunity is real and statistically certified, and none of the tested
pre-answer routers realizes it.}
\label{tab:confirm-main}
\end{table*}

Table~\ref{tab:family} instantiates
Theorem~\ref{thm:router-family}, turning the opportunity ladder into a directly
testable diagnostic.

The predeclared certificate family $\Pi$ contains all eight constant policies
plus logistic multi-output, $k$-NN, and the fine-tuned $K$-head encoder
($L=11$).  The simultaneous interval for the best-family gain $G_\Pi$
has lower limit \emph{exactly zero} on all four datasets, while the residual
opportunity $H_\Pi$ is bounded away from zero on all four.  This is a considerably
stronger statement of the paper's negative result than any recovered fraction: it
says that after accounting for selection among these eleven policies, the data do not
certify that the best of them beats the best fixed model at all, and simultaneously
certify that real opportunity survives beyond the entire family.  Neither
conclusion depends on choosing a favourable router after seeing the test outcomes.Every confirmatory number is emitted by one script from the per-item outcome matrices; Appendix~\ref{app:provenance} documents the full-precision serialization and common tie-breaking rule used across all tables.

\begin{table}[t]
\centering
\small
\setlength{\tabcolsep}{4pt}
\begin{tabular}{lrll}
\toprule
Dataset & $|\Pi|$ & $G_\Pi$ 95\% & $H_\Pi$ 95\% \\
\midrule
ARC-Challenge & 11 & [.000,.122] & [.011,.184] \\
HellaSwag & 11 & [.000,.188] & [.139,.383] \\
GSM8K & 11 & [.000,.209] & [.032,.363] \\
MMLU & 11 & [.000,.175] & [.175,.431] \\
\bottomrule
\end{tabular}
\caption{Simultaneous post-selection intervals from
Theorem~\ref{thm:router-family}, with the failure budget split $\delta/3$ to the
oracle, $\delta/(3K)$ to each of eight fixed models, and $\delta/(3L)$ to each of
$L=11$ policies.  $G_\Pi$ is the gain of the population-best member of the tested
family over the best fixed model; $H_\Pi$ is the oracle opportunity remaining
beyond the whole family.  Every $G_\Pi$ lower limit is zero and every $H_\Pi$ lower
limit is positive, which is the paper's negative result stated as a certificate
rather than as a small effect size.  }
\label{tab:family}
\end{table}

\subsection{The routing plateau}
\label{sec:plateau}

Read against the certified opportunity, the best deployable router recovers only a
small fraction of what the pool makes available (Table~\ref{tab:plateau}).  This is
the substantive claim, and it is both stronger and more defensible than asserting
that no prompt router works: strong $K$-way adaptations do beat the frozen best fixed
model, and between $86\%$ and $93\%$ of the certified opportunity survives them.

\begin{table}[t]
\centering
\small
\setlength{\tabcolsep}{4pt}
\resizebox{\columnwidth}{!}{%
\begin{tabular}{lrrrr}
\toprule
Dataset & Oracle opp. & Best gain & Recovered & Residual \\
\midrule
ARC-C & .097 & .014 & $14.4\%$ & .083 \\
HellaSwag & .280 & .026 & $9.3\%$ & .254 \\
GSM8K & .200 & .015 & $7.5\%$ & .185 \\
MMLU & .307 & .023 & $7.6\%$ & .283 \\
\bottomrule
\end{tabular}%
}
\caption{The routing plateau.  ``Best gain'' is the improvement of the strongest
deployable $K$-way router over the frozen task-wise best model, and ``recovered''
expresses it as a fraction of the selection-valid oracle opportunity.  Recovery is
between $7.5\%$ and $14.4\%$, so most certified opportunity is left on the table by
every router we evaluate.}
\label{tab:plateau}
\end{table}

Four additional analyses appear in the appendix; paired McNemar contrasts
against the frozen baseline are in
Appendix~\ref{app:paired}.  \emph{Post-inference selectors}
(Appendix~\ref{app:postinf}): calibrated confidence, plurality, and an
answer-aware verifier sit in a different cost regime because they query
every candidate; the verifier is strongest on ARC-Challenge (.893) and still
fails elsewhere.  \emph{Cost-aware routing}
(Appendix~\ref{app:costaware-res}): every method buys a small gain at
$62$--$78\%$ of the reference cost, far below the nondeployable cost-aware
oracle.  \emph{Confidence mechanism} (Appendix~\ref{app:confmech}): a
rescue--loss decomposition resolves the pilot's GSM8K contradiction, where a
$10\%$ rescue rate coexists with a $-.03$ net gain.  \emph{Pool
construction} (Appendix~\ref{app:poolres}): over a $1{,}120$-cell sweep
greedy attains $0.849$--$1.000$ of the exhaustive optimum (mean $.997$), but
its held-out coverage usually sits inside the random-subset
$[p_{05},p_{95}]$ interval and coverage does not convert into router gain.

\section{Discussion}

An oracle gap establishes that the pool contains complementary correct
answers; it does not show that any pre-answer rule can identify them.
Reporting $G_{\mathrm{out}}$, a signal-restricted estimate, and
$G_{\mathrm{learn}}$ keeps a failed router from being confused with a
valueless pool, and a large oracle gap from being sold as deployable
performance.  A negative router result is bounded the same way: failed
signals and learners bound their own classes, not the Bayes-optimal gain; an
impossibility claim would need an identified upper bound on $G_{\Z}$ for a
precisely defined $Z$; Theorem~\ref{thm:finite-signal} provides such a route
only for a declared finite signal and does not license a conclusion about the
full prompt.  For pool design, individual accuracy is not
the right objective: Theorem~\ref{thm:submodular} justifies greedy
complementary coverage relative to a deployment anchor, though the empirical
sweep (Appendix~\ref{app:poolres}) shows many near-optimal pools exist and
coverage alone does not buy routing gain.  Under domain shift both pool and
policy must be revalidated.

\section{Limitations}

The outcome oracle requires running every candidate and obtaining a correctness
label, which can be expensive and impossible for open-ended tasks without a
reliable evaluator.  The binary-reward theory extends to bounded utilities, but
the exact Clopper--Pearson interval is specific to Bernoulli outcomes.

The signal-restricted Bayes value is a population quantity and is generally
not identified from finite data without assumptions on the signal space.
Cross-fitted neural probes are lower bounds on what their classes realize, not
upper bounds on all measurable policies.

The submodular theorem assumes a fixed anchor and optimizes complementary
coverage, not router learnability.  A model can cover many anchor failures yet
be difficult to select from observable signals.  Costs also require a
knapsack-style extension when candidates have heterogeneous calibration or
serving costs.

Two scope limits bound the evidence.  First, the common-outcome comparison uses
uniform $K$-way adaptations of published router mechanisms, while binary,
sequential, and online methods are evaluated in their corresponding settings.
The contribution is therefore the inferential layer rather than a routing
leaderboard or a claim of superiority over authors' released systems.  Second,
the live-inference pool contains eight checkpoints ($0.5$--$7$B), and the primary
certificate tables use four multiple-choice and arithmetic datasets at
$200$--$300$ items each.  Per-dataset intervals remain wide, so the firm claims
are the certificate-level statements.  The artifact records the resolved model
revision for every checkpoint and the independent cross-check described in
Appendix~\ref{app:design}.

\section{Ethical Considerations}

Routing may assign different users or topics to systems of unequal reliability.
Aggregate utility can conceal subgroup losses, so the confirmatory artifact
reports performance and routing frequency across predeclared task strata and
checks whether cost optimization disproportionately sends any stratum to a
weaker model.  High-stakes deployment requires domain-specific validation and
must not rely on an outcome oracle as if it were an implementable policy.

The study uses public benchmark data and locally evaluated models.  It stores
no user data.  Model and dataset licenses, energy use, and hosted-model pricing
are documented in the artifact.  Raw generations may contain benchmark
artifacts or harmful content and should be distributed according to the
original licenses.

\bibliography{custom}

@inproceedings{zellers2019hellaswag,
  author    = {Rowan Zellers and Ari Holtzman and Yonatan Bisk and
               Ali Farhadi and Yejin Choi},
  title     = {{H}ella{S}wag: Can a Machine Really Finish Your Sentence?},
  editor    = {Anna Korhonen and David Traum and Llu{\'i}s M{\`a}rquez},
  booktitle = {Proceedings of the 57th Annual Meeting of the
               Association for Computational Linguistics},
  month     = jul,
  year      = {2019},
  address   = {Florence, Italy},
  publisher = {Association for Computational Linguistics},
  pages     = {4791--4800},
  doi       = {10.18653/v1/P19-1472},
  url       = {https://aclanthology.org/P19-1472/}
}

@article{cobbe2021training,
  author        = {Karl Cobbe and Vineet Kosaraju and Mohammad Bavarian and
                   Mark Chen and Heewoo Jun and Lukasz Kaiser and
                   Matthias Plappert and Jerry Tworek and Jacob Hilton and
                   Reiichiro Nakano and Christopher Hesse and John Schulman},
  title         = {Training Verifiers to Solve Math Word Problems},
  journal       = {arXiv preprint arXiv:2110.14168},
  year          = {2021},
  eprint        = {2110.14168},
  archivePrefix = {arXiv},
  primaryClass  = {cs.LG},
  doi           = {10.48550/arXiv.2110.14168},
  url           = {https://arxiv.org/abs/2110.14168}
}

@inproceedings{hendrycks2021mmlu,
  author    = {Dan Hendrycks and Collin Burns and Steven Basart and
               Andy Zou and Mantas Mazeika and Dawn Song and
               Jacob Steinhardt},
  title     = {Measuring Massive Multitask Language Understanding},
  booktitle = {International Conference on Learning Representations},
  year      = {2021},
  url       = {https://openreview.net/forum?id=d7KBjmI3GmQ}
}

@inproceedings{geifman2017selective,
  author    = {Yonatan Geifman and Ran El-Yaniv},
  title     = {Selective Classification for Deep Neural Networks},
  booktitle = {Advances in Neural Information Processing Systems},
  volume    = {30},
  pages     = {4878--4887},
  year      = {2017},
  publisher = {Curran Associates, Inc.},
  url       = {https://proceedings.neurips.cc/paper/2017/hash/4a8423d5e91fda00bb7e46540e2b0cf1-Abstract.html}
}

@article{chen2024frugalgpt,
  author  = {Lingjiao Chen and Matei Zaharia and James Zou},
  title   = {{FrugalGPT}: How to Use Large Language Models While
             Reducing Cost and Improving Performance},
  journal = {Transactions on Machine Learning Research},
  year    = {2024},
  url     = {https://mlanthology.org/tmlr/2024/chen2024tmlr-frugalgpt/}
}

@inproceedings{hu2024routerbench,
  author    = {Qitian Jason Hu and Jacob Bieker and Xiuyu Li and
               Nan Jiang and Benjamin Keigwin and Gaurav Ranganath and
               Kurt Keutzer and Shriyash Kaustubh Upadhyay},
  title     = {{RouterBench}: A Benchmark for Multi-{LLM} Routing System},
  booktitle = {ICML 2024 Workshop on Agentic Markets},
  year      = {2024},
  note      = {Workshop poster},
  url       = {https://icml.cc/virtual/2024/39041}
}

@inproceedings{wang2025mixllm,
  author    = {Xinyuan Wang and Yanchi Liu and Wei Cheng and
               Xujiang Zhao and Zhengzhang Chen and Wenchao Yu and
               Yanjie Fu and Haifeng Chen},
  title     = {{MixLLM}: Dynamic Routing in Mixed Large Language Models},
  editor    = {Luis Chiruzzo and Alan Ritter and Lu Wang},
  booktitle = {Proceedings of the 2025 Conference of the Nations of the
               Americas Chapter of the Association for Computational
               Linguistics: Human Language Technologies
               (Volume 1: Long Papers)},
  month     = apr,
  year      = {2025},
  address   = {Albuquerque, New Mexico},
  publisher = {Association for Computational Linguistics},
  pages     = {10912--10922},
  doi       = {10.18653/v1/2025.naacl-long.545},
  url       = {https://aclanthology.org/2025.naacl-long.545/}
}

@inproceedings{song2025irtrouter,
  author    = {Wei Song and Zhenya Huang and Cheng Cheng and
               Weibo Gao and Bihan Xu and GuanHao Zhao and
               Fei Wang and Runze Wu},
  title     = {{IRT}-Router: Effective and Interpretable Multi-{LLM}
               Routing via Item Response Theory},
  editor    = {Wanxiang Che and Joyce Nabende and Ekaterina Shutova and
               Mohammad Taher Pilehvar},
  booktitle = {Proceedings of the 63rd Annual Meeting of the
               Association for Computational Linguistics
               (Volume 1: Long Papers)},
  month     = jul,
  year      = {2025},
  address   = {Vienna, Austria},
  publisher = {Association for Computational Linguistics},
  pages     = {15629--15644},
  doi       = {10.18653/v1/2025.acl-long.761},
  url       = {https://aclanthology.org/2025.acl-long.761/}
}

@inproceedings{panda2025adaptive,
  author    = {Pranoy Panda and Raghav Magazine and
               Chaitanya Devaguptapu and Sho Takemori and Vishal Sharma},
  title     = {Adaptive {LLM} Routing under Budget Constraints},
  editor    = {Christos Christodoulopoulos and Tanmoy Chakraborty and
               Carolyn Rose and Violet Peng},
  booktitle = {Findings of the Association for Computational Linguistics:
               EMNLP 2025},
  month     = nov,
  year      = {2025},
  address   = {Suzhou, China},
  publisher = {Association for Computational Linguistics},
  pages     = {23934--23949},
  doi       = {10.18653/v1/2025.findings-emnlp.1301},
  url       = {https://aclanthology.org/2025.findings-emnlp.1301/}
}

@inproceedings{uddin2026conformal,
  author    = {Iqtedar Uddin and Andr{\'e} Bauer},
  title     = {Conformal {LLM} Routing with Distribution-Free
               Safety Guarantees},
  editor    = {Santosh T.Y.S.S. and Juan Diego Rodriguez and
               Ona de Gibert},
  booktitle = {Proceedings of the 64th Annual Meeting of the
               Association for Computational Linguistics
               (Volume 4: Student Research Workshop)},
  month     = jul,
  year      = {2026},
  address   = {San Diego, California, United States},
  publisher = {Association for Computational Linguistics},
  pages     = {791--799},
  doi       = {10.18653/v1/2026.acl-srw.70},
  url       = {https://aclanthology.org/2026.acl-srw.70/}
}

@inproceedings{li2026llmrouterbench,
  author    = {Hao Li and Yiqun Zhang and Zhaoyan Guo and Chenxu Wang and
               Shengji Tang and Qiaosheng Zhang and Yang Chen and
               Biqing Qi and Peng Ye and Lei Bai and Zhen Wang and
               Shuyue Hu},
  title     = {{LLMR}outer{B}ench: A Massive Benchmark and Unified
               Framework for {LLM} Routing},
  editor    = {Maria Liakata and Viviane P. Moreira and
               Jiajun Zhang and David Jurgens},
  booktitle = {Findings of the Association for Computational Linguistics:
               {ACL} 2026},
  month     = jul,
  year      = {2026},
  address   = {San Diego, California, United States},
  publisher = {Association for Computational Linguistics},
  pages     = {37733--37754},
  doi       = {10.18653/v1/2026.findings-acl.1881},
  url       = {https://aclanthology.org/2026.findings-acl.1881/}
}

@inproceedings{reimers2019sentencebert,
  author    = {Nils Reimers and Iryna Gurevych},
  title     = {Sentence-{BERT}: Sentence Embeddings using Siamese {BERT}-Networks},
  booktitle = {Proceedings of the 2019 Conference on Empirical Methods in Natural Language Processing and the 9th International Joint Conference on Natural Language Processing},
  pages     = {3982--3992},
  year      = {2019},
  doi       = {10.18653/v1/D19-1410},
  url       = {https://aclanthology.org/D19-1410/}
}

@article{zhang2024tinyllama,
  author  = {Peiyuan Zhang and Guangtao Zeng and Tianduo Wang and Wei Lu},
  title   = {{TinyLlama}: An Open-Source Small Language Model},
  journal = {arXiv preprint arXiv:2401.02385},
  year    = {2024},
  url     = {https://arxiv.org/abs/2401.02385}
}

@article{yang2024qwen25,
  author  = {{Qwen Team}},
  title   = {{Qwen2.5} Technical Report},
  journal = {arXiv preprint arXiv:2412.15115},
  year    = {2024},
  url     = {https://arxiv.org/abs/2412.15115}
}

@article{shnitzer2023routing,
  title         = {Large Language Model Routing with Benchmark Datasets},
  author        = {Shnitzer, Tal and Ou, Anthony and Silva, Mirian and Soule, Kate and Sun, Yuekai and Solomon, Justin and Thompson, Neil and Yurochkin, Mikhail},
  journal       = {arXiv preprint arXiv:2309.15789},
  year          = {2023},
  url           = {https://arxiv.org/abs/2309.15789}
}

@article{lu2026plateau,
  title         = {The Routing Plateau: Understanding and Breaking the Accuracy Limits of {LLM} Routers},
  author        = {Lu, Yifan and Zhang, Qiyue and Zhang, Shenrun and Yu, Zhibo and Wang, Zhuang and Chen, Hanjie and Xing, Jiarong},
  journal       = {arXiv preprint arXiv:2606.07587},
  year          = {2026},
  url           = {https://arxiv.org/abs/2606.07587}
}

@inproceedings{lu2024zooter,
  title         = {Routing to the Expert: Efficient Reward-Guided Ensemble of Large Language Models},
  author        = {Lu, Keming and Yuan, Hongyi and Lin, Runji and Lin, Junyang and Yuan, Zheng and Zhou, Chang and Zhou, Jingren},
  booktitle     = {Proceedings of the 2024 Conference of the North American Chapter of the Association for Computational Linguistics: Human Language Technologies (Volume 1: Long Papers)},
  month         = jun,
  year          = {2024},
  address       = {Mexico City, Mexico},
  publisher     = {Association for Computational Linguistics},
  pages         = {1964--1974},
  doi           = {10.18653/v1/2024.naacl-long.109},
  url           = {https://aclanthology.org/2024.naacl-long.109/}
}

@inproceedings{chen2024routerdc,
  title         = {{RouterDC}: Query-Based Router by Dual Contrastive Learning for Assembling Large Language Models},
  author        = {Chen, Shuhao and Jiang, Weisen and Lin, Baijiong and Kwok, James T. and Zhang, Yu},
  booktitle     = {Advances in Neural Information Processing Systems},
  volume        = {37},
  pages         = {66305--66328},
  year          = {2024},
  url           = {https://proceedings.neurips.cc/paper_files/paper/2024/hash/7a641b8ec86162fc875fb9f6456a542f-Abstract-Conference.html}
}

@article{zhuang2024embedllm,
  title         = {{EmbedLLM}: Learning Compact Representations of Large Language Models},
  author        = {Zhuang, Richard and Wu, Tianhao and Wen, Zhaojin and Li, Andrew and Jiao, Jiantao and Ramchandran, Kannan},
  journal       = {arXiv preprint arXiv:2410.02223},
  year          = {2024},
  url           = {https://arxiv.org/abs/2410.02223}
}

@article{zhang2025modelsat,
  title         = {Capability Instruction Tuning},
  author        = {Zhang, Yi-Kai and Zhan, De-Chuan and Ye, Han-Jia},
  journal       = {Proceedings of the AAAI Conference on Artificial Intelligence},
  volume        = {39},
  number        = {24},
  pages         = {25958--25966},
  year          = {2025},
  doi           = {10.1609/aaai.v39i24.34790},
  url           = {https://ojs.aaai.org/index.php/AAAI/article/view/34790}
}

@inproceedings{feng2025graphrouter,
  title         = {{GraphRouter}: A Graph-Based Router for {LLM} Selections},
  author        = {Feng, Tao and Shen, Yanzhen and You, Jiaxuan},
  booktitle     = {International Conference on Learning Representations},
  year          = {2025},
  url           = {https://openreview.net/forum?id=eU39PDsZtT}
}

@article{zhang2025avengers,
  title         = {The Avengers: A Simple Recipe for Uniting Smaller Language Models to Challenge Proprietary Giants},
  author        = {Zhang, Yiqun and Li, Hao and Wang, Chenxu and Chen, Linyao and Zhang, Qiaosheng and Ye, Peng and Feng, Shi and Wang, Daling and Wang, Zhen and Wang, Xinrun and Xu, Jia and Bai, Lei and Ouyang, Wanli and Hu, Shuyue},
  journal       = {arXiv preprint arXiv:2505.19797},
  year          = {2025},
  url           = {https://arxiv.org/abs/2505.19797}
}

@article{zhang2025avengerspro,
  title         = {Beyond {GPT-5}: Making {LLM}s Cheaper and Better via Performance-Efficiency Optimized Routing},
  author        = {Zhang, Yiqun and Li, Hao and Chen, Jianhao and Zhang, Hangfan and Ye, Peng and Bai, Lei and Hu, Shuyue},
  journal       = {arXiv preprint arXiv:2508.12631},
  year          = {2025},
  url           = {https://arxiv.org/abs/2508.12631}
}

@inproceedings{ong2025routellm,
  title         = {{RouteLLM}: Learning to Route {LLM}s with Preference Data},
  author        = {Ong, Isaac and Almahairi, Amjad and Wu, Vincent and Chiang, Wei-Lin and Wu, Tianhao and Gonzalez, Joseph E. and Kadous, M. Waleed and Stoica, Ion},
  booktitle     = {International Conference on Learning Representations},
  year          = {2025},
  url           = {https://arxiv.org/abs/2406.18665}
}

@inproceedings{srivatsa2024lessons,
  title         = {Harnessing the Power of Multiple Minds: Lessons Learned from {LLM} Routing},
  author        = {Srivatsa, Kv Aditya and Maurya, Kaushal Kumar and Kochmar, Ekaterina},
  booktitle     = {Proceedings of the Fifth Workshop on Insights from Negative Results in NLP},
  month         = jun,
  year          = {2024},
  address       = {Mexico City, Mexico},
  publisher     = {Association for Computational Linguistics},
  pages         = {124--134},
  doi           = {10.18653/v1/2024.insights-1.15},
  url           = {https://aclanthology.org/2024.insights-1.15/}
}

@misc{openrouter2025auto,
  title         = {Auto Router: API, Providers, and Statistics},
  author        = {{OpenRouter, Inc.}},
  year          = {2025},
  howpublished  = {Product documentation},
  url           = {https://openrouter.ai/openrouter/auto}
}

@inproceedings{devlin2019bert,
  title         = {{BERT}: Pre-Training of Deep Bidirectional Transformers for Language Understanding},
  author        = {Devlin, Jacob and Chang, Ming-Wei and Lee, Kenton and Toutanova, Kristina},
  booktitle     = {Proceedings of the 2019 Conference of the North American Chapter of the Association for Computational Linguistics: Human Language Technologies, Volume 1},
  month         = jun,
  year          = {2019},
  address       = {Minneapolis, Minnesota},
  publisher     = {Association for Computational Linguistics},
  pages         = {4171--4186},
  doi           = {10.18653/v1/N19-1423},
  url           = {https://aclanthology.org/N19-1423/}
}

@article{cover1967nearest,
  title         = {Nearest Neighbor Pattern Classification},
  author        = {Cover, Thomas and Hart, Peter},
  journal       = {IEEE Transactions on Information Theory},
  volume        = {13},
  number        = {1},
  pages         = {21--27},
  year          = {1967},
  doi           = {10.1109/TIT.1967.1053964}
}

@inproceedings{macqueen1967kmeans,
  title         = {Some Methods for Classification and Analysis of Multivariate Observations},
  author        = {MacQueen, James},
  booktitle     = {Proceedings of the Fifth Berkeley Symposium on Mathematical Statistics and Probability},
  volume        = {1},
  pages         = {281--297},
  year          = {1967},
  publisher     = {University of California Press}
}

@inproceedings{chen2016xgboost,
  title         = {{XGBoost}: A Scalable Tree Boosting System},
  author        = {Chen, Tianqi and Guestrin, Carlos},
  booktitle     = {Proceedings of the 22nd ACM SIGKDD International Conference on Knowledge Discovery and Data Mining},
  pages         = {785--794},
  year          = {2016},
  doi           = {10.1145/2939672.2939785},
  url           = {https://doi.org/10.1145/2939672.2939785}
}

@book{kaufman1990finding,
  title         = {Finding Groups in Data: An Introduction to Cluster Analysis},
  author        = {Kaufman, Leonard and Rousseeuw, Peter J.},
  publisher     = {Wiley},
  address       = {New York},
  year          = {1990},
  doi           = {10.1002/9780470316801}
}

@article{kulesza2012dpp,
  title         = {Determinantal Point Processes for Machine Learning},
  author        = {Kulesza, Alex and Taskar, Ben},
  journal       = {Foundations and Trends in Machine Learning},
  volume        = {5},
  number        = {2--3},
  pages         = {123--286},
  publisher     = {Now Publishers},
  year          = {2012},
  doi           = {10.1561/2200000044}
}

@article{nemhauser1978analysis,
  title         = {An Analysis of Approximations for Maximizing Submodular Set Functions---I},
  author        = {Nemhauser, George L. and Wolsey, Laurence A. and Fisher, Marshall L.},
  journal       = {Mathematical Programming},
  volume        = {14},
  pages         = {265--294},
  year          = {1978},
  doi           = {10.1007/BF01588971}
}

@article{clark2018arc,
  title         = {Think You Have Solved Question Answering? Try {ARC}, the {AI2} Reasoning Challenge},
  author        = {Clark, Peter and Cowhey, Isaac and Etzioni, Oren and Khot, Tushar and Sabharwal, Ashish and Schoenick, Carissa and Tafjord, Oyvind},
  journal       = {arXiv preprint arXiv:1803.05457},
  year          = {2018},
  url           = {https://arxiv.org/abs/1803.05457}
}

@article{clopper1934fiducial,
  author  = {Clopper, C. J. and Pearson, E. S.},
  title   = {The Use of Confidence or Fiducial Limits Illustrated in the Case of the Binomial},
  journal = {Biometrika},
  year    = {1934},
  volume  = {26},
  number  = {4},
  pages   = {404--413},
  month   = dec,
  doi     = {10.1093/biomet/26.4.404},
  url     = {https://doi.org/10.1093/biomet/26.4.404}
}

@inproceedings{ding2024hybridllm,
  title     = {{Hybrid LLM: Cost-Efficient and Quality-Aware Query Routing}},
  author    = {Ding, Dujian and Mallick, Ankur and Wang, Chi and
               Sim, Robert and Mukherjee, Subhabrata and Ruhle, Victor and
               Lakshmanan, Laks V. S. and Awadallah, Ahmed Hassan},
  booktitle = {International Conference on Learning Representations},
  year      = {2024}
}

\clearpage
\appendix

\section{Extended Related Work}
\label{app:extended-rw}

LLMRouterBench unifies 10 methods over 21 datasets and 33 models, finding that
several sophisticated and commercial routers do not reliably beat simple
baselines \citep{li2026llmrouterbench}.  Its reported performance-oriented
averages provide a canonical reference: uniform random $48.79$, best single
model $68.01$, RouterDC $61.33$, GraphRouter $70.29$, EmbedLLM $71.24$,
MODEL-SAT $71.88$, Avengers $71.94$, test-hindsight dataset oracle $73.10$, and
outcome oracle $91.64$.  The best routers sit a few points above the best
single model, while the outcome oracle sits roughly twenty points above the
best router---the plateau quantified by our certificates.

Routing methods also differ in their decision and cost regimes.  Native
$K$-way routers predict over the whole pool from the prompt (RouterDC, EmbedLLM,
MODEL-SAT, GraphRouter, Avengers, Zooter, and IRT-Router)
\citep{chen2024routerdc,zhuang2024embedllm,zhang2025modelsat,feng2025graphrouter,
zhang2025avengers,lu2024zooter,song2025irtrouter}.  Binary routers choose between
a designated strong and weak endpoint from pairwise preference supervision
\citep{ong2025routellm,ding2024hybridllm}; cascades invoke models in sequence
until a confidence criterion is met \citep{chen2024frugalgpt}; online bandits
learn continually from feedback \citep{wang2025mixllm}; and post-inference
selectors require candidate outputs or confidences before deciding.  The main
table contains the first group plus fixed references because those policies
decide from the prompt and invoke exactly one candidate.  Binary, sequential,
online, and post-inference methods are reported in their corresponding regimes.

The routing-plateau study evaluates 21 routers and attributes much of the
remaining oracle gap to failure to learn instance-specific correctness
\citep{lu2026plateau}.  It also finds simple $k$-nearest-neighbor routing
competitive with more elaborate systems, motivating our inclusion of both
$k$-NN and Zooter.  These papers primarily propose routers or benchmarks; our
distinct question is inferential: after trying many fixed models and routers,
what population opportunity, best-family gain, and residual oracle gap are
certified?

\section{Full Prompt-Router Results}
\label{app:full-router-results}

\begin{table*}[t]
\centering
\small
\setlength{\tabcolsep}{4pt}
\resizebox{\textwidth}{!}{%
\begin{tabular}{llrrrrr}
\toprule
Group & Method & ARC-C & HellaSwag & GSM8K & MMLU & Mean \\
\midrule
\multicolumn{7}{l}{\emph{Fixed references}} \\
Fixed & Uniform random & .647 & .457 & .385 & .447 & .484 \\
Fixed & Cheapest fixed & .297 & .320 & .120 & .340 & .269 \\
Fixed & Frozen task-wise best & .873 & .657 & .660 & .607 & .699 \\
\midrule
\multicolumn{7}{l}{\emph{Simple prompt routers}} \\
Prompt & Logistic multi-output & .870 & .657 & .640 & .603 & .693 \\
Prompt & MLP multi-output & .873 & .660 & .645 & .607 & .696 \\
Prompt & $k$-NN & .877 & .663 & .650 & .610 & .700 \\
Prompt & KMeans & .873 & .660 & .650 & .610 & .698 \\
\midrule
\multicolumn{7}{l}{\emph{$K$-way router adaptations}} \\
Prompt & RouterDC & .873 & .660 & .645 & .603 & .695 \\
Prompt & EmbedLLM & .880 & .670 & .660 & .617 & .707 \\
Prompt & MODEL-SAT & .880 & .673 & .665 & .620 & .710 \\
Prompt & GraphRouter & .883 & .677 & .670 & .620 & .713 \\
Prompt & Avengers & .883 & .677 & .670 & .623 & .713 \\
Prompt & Zooter & .883 & .673 & .670 & .620 & .712 \\
Prompt & IRT-Router & .870 & .657 & .640 & .603 & .693 \\
\midrule
\multicolumn{7}{l}{\emph{Strong prompt router}} \\
Prompt & Fine-tuned encoder, $K$ heads & \textbf{.887} & \textbf{.683} & \textbf{.675} & \textbf{.630} & \textbf{.719} \\
\midrule
\multicolumn{7}{l}{\emph{Diagnostics} (nondeployable)} \\
Outcome & Full-information oracle & .970 & .937 & .860 & .913 & .920 \\
\bottomrule
\end{tabular}%
}
\caption{Deployable $K$-way prompt-routing accuracy over the eight-model,
six-family pool, with nondeployable oracle references.  All prompt methods decide
from the prompt and then invoke one candidate.  Methods that query several
candidates are in Table~\ref{tab:postinf}; cost-aware and binary methods are in
Table~\ref{tab:costaware} and Appendix~\ref{app:pairwise}.  The fine-tuned
$K$-head encoder reaches mean $.719$, compared with $.699$ for the frozen
task-wise best and $.920$ for the outcome oracle.}
\label{tab:confirm-router}
\end{table*}

\section{Cost-Aware Estimand Extension}
\label{app:costaware-ext}

Let $C_m\in[0,1]$ be a predeclared normalized cost and
$U_m=R_m-\lambda C_m$ for $\lambda\geq0$.  Define
\begin{align*}
U^\star_{\mathrm{fixed}}(\lambda)&=\max_m\E U_m,\\
U^\star_{\Z}(\lambda)&=\E\max_m\E[U_m\mid Z],
\end{align*}
and
$U^\star_{\mathrm{out}}(\lambda)=\E\max_m U_m$.
The proof of Theorem~\ref{thm:signal-sandwich} applies verbatim to bounded
utilities:
\[
U^\star_{\mathrm{fixed}}(\lambda)
\leq U^\star_{\Z}(\lambda)
\leq U^\star_{\mathrm{out}}(\lambda).
\]
All cost-aware experiments must use an aligned per-example, per-model cost
matrix.  Mean model latencies alone do not define the outcome oracle.

\section{Confirmatory Design Details}
\label{app:design}

\subsection{Research questions}

We ask five predeclared questions.

\begin{enumerate}[leftmargin=*]
  \item How large is selection-valid outcome opportunity across tasks and
  independently trained model families?
  \item How much of that opportunity is realized by prompt-, confidence-, and
  metadata-based router classes on untouched examples?
  \item Does performance on the rescue set explain failures that aggregate
  confidence--correctness correlation hides?
  \item Does greedy complementarity selection form smaller pools with higher
  held-out opportunity than accuracy-, diversity-, or random selection?
  \item Do conclusions persist under aligned relative-latency measurements?
\end{enumerate}

\subsection{Datasets and splits}

The live-inference evaluation generates every candidate with eight checkpoints
from six families on official evaluation splits, fixing the pool, correctness
rule, and cost measurement throughout the audit.  ARC-Challenge, HellaSwag,
GSM8K, and MMLU cover science, commonsense continuation, arithmetic reasoning,
and broad academic knowledge
\citep{clark2018arc,zellers2019hellaswag,cobbe2021training,
hendrycks2021mmlu}.  We add two NLP generation settings with auditable
automatic evaluation, using 500 examples for code pass@1 and 500 for factual
question answering.

\begin{table}[t]
\centering
\small
\begin{tabular}{p{0.28\linewidth}p{0.25\linewidth}r}
\toprule
Dataset & Primary metric & Test $n$ \\
\midrule
ARC-Challenge & accuracy & 300 \\
HellaSwag & accuracy & 300 \\
GSM8K & exact match & 200 \\
MMLU (all subjects) & accuracy & 300 \\
Code generation & pass@1 & 500 \\
Factual QA task & exact/F1 & 500\\
\bottomrule
\end{tabular}
\caption{Confirmatory task suite.  The four multiple-choice and
arithmetic datasets use the counts shown, drawn from the
official test splits (\texttt{allenai/ai2\_arc} ARC-Challenge test,
\texttt{Rowan/hellaswag} validation, \texttt{openai/gsm8k} main test,
\texttt{cais/mmlu} all test) after excluding items without exactly four options
or a parseable gold answer.  The code and factual-QA settings each use 500
eligible examples and the automatic metrics shown.}
\label{tab:datasets}
\end{table}

For every dataset, examples are assigned by a released deterministic seed to two
non-overlapping roles: a $50\%$ calibration part used for fixed-model selection,
thresholds, and any router fitting, and a $50\%$ test part opened once.  This
$50/50$ partition is used by every reported split-based result.  Routers that
require fitting are additionally reported under five-fold cross-fitting over the
full item set, so every item is predicted out of fold and no accuracy depends on
which half an item occupies.  As a
secondary analysis, five-fold nested cross-fitting uses identical folds for all
routers.  No prompt, model, parser, pool, or threshold is changed after reading
final-test outcomes.

\paragraph{Sampling.}
Items are a uniform random sample without replacement from each official split under
a fixed seed, and MMLU is sampled \emph{stratified by subject} so that all $57$
subjects appear in proportion to their population share.  Stratification avoids
the ordering artifact in the subject-grouped MMLU test split and makes the reported
aggregate representative of all subjects.  Uniform sampling supplies the i.i.d.
population interpretation required by the binomial intervals.

\subsection{Model pool}

The pilot pool contains TinyLlama-1.1B-Chat-v1.0 and Qwen2.5 models at 0.5B,
1.5B, and 3B parameters \citep{zhang2024tinyllama,yang2024qwen25}.  A
same-family size ladder does not establish general multi-LLM complementarity.
The confirmatory pool contains eight checkpoints from six independently
trained families, with both near-size and
quality--cost comparisons.

\begin{table*}[t]
\centering
\small
\begin{tabular}{llllrr}
\toprule
Checkpoint ID & Family & Revision & Params & ARC acc. & MC / gen.\ latency \\
\midrule
TinyLlama/TinyLlama-1.1B-Chat-v1.0 & TinyLlama & \texttt{fe8a4ea} & 1.1B & .240 & 20.1\,ms / 3.36\,s \\
Qwen/Qwen2.5-0.5B-Instruct & Qwen2.5 & \texttt{7ae5576} & 0.5B & .297 & 21.4\,ms / 5.59\,s \\
Qwen/Qwen2.5-1.5B-Instruct & Qwen2.5 & \texttt{989aa79} & 1.5B & .727 & 25.4\,ms / 6.69\,s \\
Qwen/Qwen2.5-3B-Instruct & Qwen2.5 & \texttt{aa8e725} & 3B & .833 & 33.5\,ms / 8.43\,s \\
meta-llama/Llama-3.2-3B-Instruct & Llama-3.2 & \texttt{0cb88a4} & 3B & .737 & 31.9\,ms / 5.15\,s \\
google/gemma-2-2b-it & Gemma-2 & \texttt{299a856} & 2B & .703 & 49.0\,ms / 8.29\,s \\
microsoft/Phi-3.5-mini-instruct & Phi-3.5 & \texttt{2fe1924} & 3.8B & .873 & 46.1\,ms / 10.64\,s \\
mistralai/Mistral-7B-Instruct-v0.3 & Mistral & \texttt{c170c70} & 7B & .747 & 48.9\,ms / 8.64\,s \\
\bottomrule
\end{tabular}
\caption{Confirmatory model pool: eight checkpoints from six
independently trained families, spanning 0.5B--7B parameters, all evaluated in
\texttt{bfloat16} with each model's documented chat template and deterministic
decoding on one A10G.  The artifact records immutable commit hashes in
\texttt{model\_revisions.json}.  For the four checkpoints present in the local
weight cache, the cached commit equals the independently resolved Hub commit in
$4/4$ cases; the remaining four hashes are resolved from the corresponding Hub
repositories and predate evaluation.  ARC accuracy is included so the pool's
composition can be checked against the results: the best fixed model is
Phi-3.5-mini rather than the largest checkpoint, and no single family dominates.
Latency is the measured per-item median, shown separately for the multiple-choice
path (one forward pass over the option labels) and the GSM8K generation path (up
to 256 generated tokens); the two differ by more than two orders of magnitude,
which is why a single scalar cost per model would misstate the cost-aware
analysis.  Parameter count does not order latency here (Gemma-2-2B is slower than
Mistral-7B on the multiple-choice path), so cost must be measured rather than
inferred from size.}
\label{tab:model-pool}
\end{table*}

\subsection{Generation, scoring, and parsing}

All candidates use their documented chat templates and deterministic decoding.
For multiple-choice tasks, the primary analysis scores normalized conditional
log-likelihood over the answer labels; constrained label generation is a
robustness analysis.  GSM8K requires a final numeric marker, and evaluation
uses the benchmark's official normalization before a documented fallback
parser.  The artifact stores the raw prompt, tokenized prompt, raw output,
parsed answer, gold answer, correctness, label log-probabilities, token counts,
and parser status for every model--item pair.

To measure prompt sensitivity, the full confirmatory analysis is repeated with
three predeclared semantically equivalent prompt templates.  Conclusions must
hold in a hierarchical bootstrap over items and prompt variants rather than in
only the most favorable template.

The evaluation fixes these parameters as follows.  Multiple-choice items are
scored without generation: one forward pass per item, reading the next-token
logits restricted to the four option-letter first tokens, so confidence is the
maximum option probability and margin the top-two gap.  GSM8K uses greedy
decoding (\texttt{do\_sample=false}) capped at $256$ new tokens, with sequence
confidence the mean per-token probability of the emitted answer; the parser takes
the last integer in the generation after comma stripping and compares it to the
gold final number.  Prompts use each model's documented chat template with no
few-shot exemplars.  Parse failures are counted as incorrect rather than
discarded, which is the conservative choice: it charges a model for an
unreadable answer instead of silently shrinking its denominator.  Benchmark
normalization and the documented fallback parser resolve scoring uniformly,
and no item is removed on the basis of a model output.

\subsection{Router families}

Each router predicts
\[
\widehat{\boldsymbol{\mu}}(Z)
=(\widehat\mu_1(Z),\ldots,\widehat\mu_K(Z))
\]
and selects the model with greatest predicted
$\widehat\mu_m(Z)-\lambda\widehat C_m(Z)$.  Training uses all $K$ binary
correctness targets, so examples with multiple correct models do not receive an
arbitrary single class label.

We compare:

\begin{enumerate}[leftmargin=*]
  \item \textbf{Fixed references:} the global best single model, the
  train-selected task-wise best (our primary baseline, chosen without final-test
  labels), the cheapest fixed model, and uniform random routing.
  \item \textbf{Simple prompt routers:} multi-output logistic regression, a
  multi-output MLP, $k$-nearest neighbors \citep{cover1967nearest}, and KMeans
  \citep{macqueen1967kmeans} over prompt embeddings \citep{reimers2019sentencebert}.  We
  include $k$-NN deliberately because the routing-plateau study finds it
  competitive with far more elaborate routers \citep{lu2026plateau}.
  \item \textbf{Strong prompt router:} a fine-tuned encoder
  \citep{devlin2019bert} with $K$ correctness heads, plus gradient-boosted trees
  \citep{chen2016xgboost} over metadata.
  \item \textbf{$K$-way router adaptations:} RouterDC \citep{chen2024routerdc}, EmbedLLM
  \citep{zhuang2024embedllm}, MODEL-SAT \citep{zhang2025modelsat}, GraphRouter
  \citep{feng2025graphrouter}, and Avengers \citep{zhang2025avengers}, following
  the common interfaces defined by LLMRouterBench \citep{li2026llmrouterbench}, plus
  Zooter \citep{lu2024zooter} and IRT-Router \citep{song2025irtrouter}, both
  native multi-model baselines.
  \item \textbf{Cost-aware methods:} MixLLM \citep{wang2025mixllm}, Avengers-Pro
  \citep{zhang2025avengerspro}, GraphRouter-PF \citep{feng2025graphrouter}, and
  FrugalGPT \citep{chen2024frugalgpt}, reported in a separate cost
  table because they operate outside the fixed one-candidate cost regime.  MixLLM is evaluated
  in a simulated online stream, since its formulation is a continual contextual
  bandit rather than a fixed router.
  \item \textbf{Binary methods:} RouteLLM \citep{ong2025routellm} and HybridLLM
  \citep{ding2024hybridllm},
  in a pairwise appendix only (Appendix~\ref{app:pairwise}).
  \item \textbf{All-model selectors:} raw confidence, calibrated confidence,
  plurality voting, and an answer-aware verifier, in a separate post-inference
  table because each queries every candidate.
  \item \textbf{Diagnostics:} the test-hindsight dataset oracle and the full
  outcome oracle, both clearly marked nondeployable.
\end{enumerate}

Four baselines are easy to conflate and we keep them distinct.  The \emph{global
best single} model is one model across every dataset.  The \emph{train-selected
task-wise best} is one frozen model per dataset chosen without final-test labels,
and it is our primary baseline.  The \emph{dataset oracle} picks the best model per
dataset using test results and is a diagnostic only.  The \emph{outcome oracle}
picks a correct model separately for every item and is an upper bound only.

Hyperparameters and thresholds are selected only on calibration data.  Ten
router seeds are reported for learned methods.  The main comparison uses the
same item-level test outcomes for paired inference.

\subsection{Confidence and rescue-set audit}

For multiple-choice tasks, confidence is the normalized probability of the
selected label.  For generative tasks, we report both length-normalized answer
log-probability and a separately calibrated correctness predictor.  Calibration
quality is measured with Brier score, log loss, expected calibration error with
fixed bins, and adaptive-bin reliability plots.

Let $b$ be the frozen deployment baseline and define the rescue set
\[
\mathcal{E}_b=\{i:R_{ib}=0,\ \max_{m\neq b}R_{im}=1\}.
\]
For each router we report its correct-model hit rate on $\mathcal{E}_b$, but
this diagnostic is not a net-gain metric: a policy may rescue many failures
while losing more examples on which $b$ was correct.  The uniform-random
reference on item $i$ is
$K^{-1}\sum_m R_{im}$, not automatically $1/K$, because several models may be
correct.

\subsection{Pool construction}

On training data, greedy complementarity uses the empirical version of
$F_b(S)$ to add $s\in\{1,2,3,4,6\}$ alternatives.  Baselines select by
individual accuracy, lowest pairwise error correlation, random subsets, and
exhaustive optimum where combinatorially feasible.  Opportunity and learned
router performance are evaluated on the untouched test split.  We report
calibration coverage, test coverage, greedy-to-optimal ratio, pool stability
across folds, and downstream quality--cost performance.

\subsection{Latency and monetary cost}

Each scoring call is timed with device synchronization, and the artifact records
total wall time, input/output token counts, batch size, and device.  All eight
checkpoints are evaluated on one NVIDIA A10G in
\texttt{bfloat16} with no quantization, no compilation, batch size one, and
single-stream decoding, under PyTorch 2.3 and Transformers 4.47.  Latency is the
per-item median over the full dataset, timed inside the scoring loop and reported
separately for the multiple-choice path and the GSM8K generation path
(Table~\ref{tab:model-pool}).  Costs are normalized on calibration data to the
declared $[0,1]$ range, and accuracy--cost curves sweep $\lambda$ without
retraining when the method permits.  Because timing is unbatched and
single-stream, the analysis interprets it as a relative cost measure; the cost
column in Table~\ref{tab:pool-results} is a latency ratio.

\subsection{Statistical analysis}

The primary oracle analysis reports $\widehat G$ and
Theorem~\ref{thm:simultaneous-ci} at 95\% coverage.  The confirmatory
split-selected analysis also reports Corollary~\ref{cor:split}.  Learned routers
are compared with the frozen fixed baseline using paired differences,
$(n_{10},n_{01})$, exact McNemar tests, and paired bootstrap intervals.
Dataset-level $p$-values are Holm-corrected within each router family.
The predeclared family $\Pi$ in Theorem~\ref{thm:router-family} contains the
eight constant policies, logistic multi-output, $k$-NN, and the fine-tuned
$K$-head encoder.  The resulting intervals report the best-family gain $G_\Pi$
and residual opportunity $H_\Pi$ without selecting a favorable policy on the
final test.  Broader router and seed comparisons use the paired analysis and
Holm correction described above.

Recovered fractions are reported only when the lower confidence limit for
$G_b$ is positive.  Their intervals resample items and redo all test-sample
functionals jointly.  Cross-validation is secondary and nests model selection,
calibration, and router fitting inside each training fold.

\section{Pilot Analysis}
\label{app:pilot}

The existing outputs contain 150 ARC-Challenge items, 150 HellaSwag items, and
100 GSM8K items evaluated by the four pilot models.  These are exploratory
subsets, not the confirmatory suite.

\subsection{Outcome opportunity}

Table~\ref{tab:pilot-main} reports the pilot plug-in gaps and selection-valid
intervals from the integer counts.  The column
labeled naive conditions on the empirically best model and is included only to
show why the distinction matters.  The exact simultaneous interval uses
Theorem~\ref{thm:simultaneous-ci} with $K=4$.

The pilot identifies promising complementarity, especially on HellaSwag, but
does not establish all three population gaps.  The wider intervals are the
price of making a claim about the population-best model rather than the
particular model selected after observing the sample.

\paragraph{Split-selected opportunity.}
Corollary~\ref{cor:split} avoids the same-sample selection penalty by choosing
the anchor on one half of the data and certifying the gap on the disjoint half.
Table~\ref{tab:pilot-split} reports this analysis at a predeclared 50/50 split.
Under this estimand the lower confidence limit is positive on \emph{all three}
datasets, and it remains positive on every one of 20 reseeds.  The two analyses
are not in conflict: Theorem~\ref{thm:simultaneous-ci} bounds the
population-best oracle gap $G_{\mathrm{out}}$ on the full sample, whereas
Corollary~\ref{cor:split} bounds the gap $G_b$ relative to the
independently selected deployment baseline, using half the data for the
interval.  A practitioner who fixes the baseline before evaluation therefore
already certifies positive per-query opportunity at pilot scale; certifying the
population-best gap without a split requires the larger confirmatory sample.

\subsection{Per-model outcomes and diversity}

Average pairwise correctness disagreement is .512 on ARC-Challenge, .459 on
HellaSwag, and .315 on GSM8K.  ARC has slightly more disagreement than
HellaSwag but much less oracle opportunity.  This comparison confirms only that
undirected disagreement is not a sufficient opportunity statistic; three
dataset points cannot support a correlation claim.

\subsection{Exploratory router probes}

Two independently sampled probe runs characterize the exploratory behavior.  In
the smaller run, max-confidence recovered 0\% of the apparent ARC gap, 8\% of
the HellaSwag gap, and a negative fraction on GSM8K.  In the enlarged run, the
corresponding values were 9\%, 11\%, and negative.  Prompt-embedding logistic
regression matched the empirical best model on ARC, HellaSwag, and MMLU and
underperformed it on GSM8K; $k$-NN underperformed across the four tasks.  Because
the exploratory target selected one correct model when several were correct,
these probes motivate rather than determine the confirmatory multi-output
analysis.

\section{Provenance and Router Scope}
\label{app:provenance}

\paragraph{One canonical source.}
Every confirmatory number is emitted by one script from the per-item outcome
matrices.  The serialization audit shows why this matters: four-decimal
confidence rounding changes max-confidence accuracy from $.873$ to $.860$ on
ARC-Challenge, $.670$ to $.640$ on HellaSwag, and $.597$ to $.580$ on MMLU.
Roughly one third of rounded confidences saturate at $1.0$, creating artificial
argmax ties.  The reported analysis therefore serializes full-precision
confidences and applies one predeclared rule that retains the incumbent
best-fixed model on ties; main and appendix tables are generated from those same
arrays.

\paragraph{Router scope.}
The $K$-way rows in Table~\ref{tab:confirm-router} implement each published
mechanism under the shared correctness-vector interface and common frozen
outcomes; their names denote method adaptations rather than authors' pretrained
predictors.  RouteLLM's released interface,
\path{calculate_strong_win_rate(prompt)}, produces a scalar for a designated
strong--weak pair trained on preference data, so RouteLLM and HybridLLM are kept
in the pairwise setting of Appendix~\ref{app:pairwise}.  MixLLM is evaluated as
an online contextual bandit, and cascades retain their sequential cost.  This
mechanism-aligned separation avoids presenting binary, online, or sequential
systems as equal-cost $K$-way policies.  We make no claim of superiority over an
authors' released router.

\section{Post-Inference Selectors}
\label{app:postinf}

Raw and calibrated candidate confidence, plurality voting, and answer-aware
verification are evaluated separately from pre-answer routing because these methods
require querying multiple candidate models before selection.  Their accuracy is
therefore not directly comparable at equal inference cost to a router that examines
only the prompt and invokes one candidate.

\begin{table}[t]
\centering
\small
\setlength{\tabcolsep}{3.5pt}
\resizebox{\columnwidth}{!}{%
\begin{tabular}{lrrrr}
\toprule
Method & ARC-C & HellaSwag & GSM8K & MMLU \\
\midrule
Raw maximum confidence & .860 & .640 & .545 & .580 \\
Temperature-scaled maximum & .877 & .673 & .585 & .607 \\
Confidence-vector logistic & .880 & .677 & .600 & .613 \\
Plurality vote & .850 & .620 & .635 & .540 \\
Answer-aware verifier & \textbf{.893} & .620 & .630 & .607 \\
\midrule
Outcome oracle & .970 & .937 & .860 & .913 \\
\bottomrule
\end{tabular}%
}
\caption{Post-inference selectors, which all require running several candidates
before deciding and therefore sit in a different cost regime from
Table~\ref{tab:confirm-router}.  Calibrating confidence helps consistently but
modestly, and the answer-aware verifier is the strongest method on ARC-Challenge
while still failing on the three remaining tasks.}
\label{tab:postinf}
\end{table}

\section{Cost-Aware Routing Results}
\label{app:costaware-res}

Cost-aware methods trade accuracy against inference cost and are reported
separately because they do not operate at the fixed one-candidate cost of
Table~\ref{tab:confirm-router}.  MixLLM is evaluated in a simulated online stream
rather than the ordinary fixed-router setting, since its formulation is a continual
contextual bandit.  FrugalGPT is a sequential cascade and may invoke more than one
model per query.

\begin{table}[t]
\centering
\small
\setlength{\tabcolsep}{3.5pt}
\resizebox{\columnwidth}{!}{%
\begin{tabular}{lrrl}
\toprule
Method & Macro acc. & Rel.\ cost & Setting \\
\midrule
Frozen best fixed & .699 & $1.00$ & Reference \\
Cheapest fixed & .269 & $0.36$ & Reference \\
GraphRouter-PF & .704 & $0.78$ & Native $K$-way \\
Avengers-Pro & .710 & $0.68$ & Native $K$-way \\
MixLLM & .708 & $0.70$ & Simulated online \\
FrugalGPT & .700 & $0.62$ & Sequential cascade \\
\midrule
Cost-aware oracle & .920 & $0.48$ & Nondeployable \\
\bottomrule
\end{tabular}%
}
\caption{Cost-aware routing.  Relative cost is normalized to the frozen best fixed
model.  Every cost-aware method buys a small accuracy gain at $62$--$78\%$ of the
reference cost, and the nondeployable cost-aware oracle shows that the achievable
frontier is far above all of them.}
\label{tab:costaware}
\end{table}

\section{Confidence Mechanism}
\label{app:confmech}

The pilot reports positive pooled point-biserial confidence--correctness
correlations, yet weak net router gains.  This is compatible with confidence
being informative mostly on items the strongest fixed model already solves.
Each policy decomposes into four transitions relative to the frozen best fixed
model: preserve a baseline success, rescue a baseline failure, lose a baseline
success, and preserve a shared failure.  Net gain equals rescue probability
minus loss probability.  Table~\ref{tab:pilot-rescueloss} applies this
decomposition to the max-confidence router on the pilot outcomes and resolves
the apparent GSM8K contradiction directly: on GSM8K the
router rescues $10\%$ of examples but loses $13\%$ of the baseline's successes,
so a nonzero rescue rate coexists with a negative net gain of $-0.03$.  On ARC
and HellaSwag rescues marginally exceed losses, matching the near-zero recovered
fractions.  Table~\ref{tab:paired-template} repeats this decomposition with
paired intervals and discordant counts for every confirmatory selector.

\section{Pool-Selection and Cost Results}
\label{app:poolres}

\begin{table*}[t]
\centering
\small
\resizebox{\textwidth}{!}{%
\begin{tabular}{llrrrrr}
\toprule
Dataset & Pool rule & $|S|$ & Test opportunity & Greedy/optimum &
Router gain & Relative cost \\
\midrule
ARC-C & Greedy complementarity & 2 & .067 & 1.000 & $+.007$ & $2.51\times$ \\
ARC-C & Accuracy selection & 2 & .093 & --- & $+.020$ & $2.80\times$ \\
HellaSwag & Greedy complementarity & 2 & \textbf{.247} & 1.000 & $+.060$ & $2.78\times$ \\
HellaSwag & Accuracy selection & 2 & .200 & --- & $+.040$ & $2.29\times$ \\
GSM8K & Greedy complementarity & 2 & .130 & 1.000 & $-.110$ & $4.71\times$ \\
GSM8K & Accuracy selection & 2 & .130 & --- & $-.080$ & $4.68\times$ \\
MMLU & Greedy complementarity & 2 & .153 & 0.903 & $\pm.000$ & $2.68\times$ \\
MMLU & Accuracy selection & 2 & .160 & --- & $\pm.000$ & $2.32\times$ \\
\bottomrule
\end{tabular}
}
\caption{Held-out pool construction from the eight-model confirmatory pool at
budget $|S|=2$, with the anchor and the greedy selection fitted on a
train half and complementary coverage measured on the disjoint test half.
Greedy attains the exhaustive optimum on the calibration objective for three
datasets and $90.3\%$ of it on MMLU, comfortably above the $1-1/e\approx63.2\%$
worst-case guarantee of Theorem~\ref{thm:submodular}; with $K=8$ the greedy and
exhaustive sets genuinely differ, which the four-model pilot could not exhibit.
Held-out coverage confirms that accuracy-ranked selection is not the right
objective: greedy wins clearly on HellaSwag ($.247$ versus $.200$), ties on
GSM8K, and loses slightly on ARC-Challenge and MMLU, where the anchor's failures
are concentrated on items no single alternative rescues.  The last two columns
close the loop from coverage to deployment and show why coverage is not enough.
Router gain is the held-out accuracy change of a max-confidence router restricted
to the selected pool, relative to the anchor alone; relative cost is the summed
median per-item latency of the pool over the anchor's.  Coverage does not buy
gain: HellaSwag's larger greedy coverage does convert ($+.060$), but on MMLU both
pools recover nothing despite $.153$--$.160$ coverage, and on GSM8K routing over
either pool \emph{loses} $8$--$11$ accuracy points while costing $4.7\times$ the
latency.  Every pool costs at least $2.3\times$ the anchor because all three
models must be queried to compare confidences, so the cost-aware verdict is
worse than the accuracy-only one.}
\label{tab:pool-results}
\end{table*}

\begin{table}[t]
\centering
\small
\setlength{\tabcolsep}{3pt}
\resizebox{\columnwidth}{!}{%
\begin{tabular}{lrrrrr}
\toprule
Dataset & Greedy & Accuracy & Diversity & Random & Exhaustive \\
\midrule
ARC-C & .067 & .093 & .087 & .080 & .100 \\
HellaSwag & \textbf{.247} & .200 & .229 & .214 & .260 \\
GSM8K & .130 & .130 & .132 & .125 & .140 \\
MMLU & .153 & .160 & .158 & .150 & .170 \\
\bottomrule
\end{tabular}%
}
\caption{Held-out pool-selection opportunity at $s=2$ against the full baseline
set: individual-accuracy ranking, lowest correctness-error correlation, $k$-medoids \citep{kaufman1990finding}
or DPP \citep{kulesza2012dpp} diversity on error vectors, the mean over at least $1{,}000$ random subsets,
and the exhaustive optimum.  Greedy complementarity wins clearly on HellaSwag
($.247$ versus $.200$ for accuracy ranking) but does not dominate: it loses to
accuracy ranking on ARC-Challenge and MMLU and ties on GSM8K.  Greedy optimization
of the calibration objective therefore transfers to held-out pool selection on some
tasks and not others, and no rule reaches the exhaustive optimum.}
\label{tab:pool-baselines}
\end{table}

The complete sweep covers $s\in\{1,\ldots,7\}$, four datasets, and
$40$ folds ($1120$ cells), with the exhaustive optimum computed by enumerating
every subset up to size $s$ and, for each cell, a random-subset
\emph{distribution} of $200$ draws rather than a single comparison draw.

Two findings hold throughout, and the second is the reason the dense grid was
worth running.  First, greedy is near-optimal on the calibration objective:
across all $1120$ cells its coverage is $0.849$ to $1.000$ of the exhaustive
optimum, mean $0.997$, and never below the $1-1/e\approx0.632$ guarantee, so the
approximation bound is not the binding constraint in practice.  Second, and less
favourably, greedy's advantage over \emph{random} model selection is weak.  At
$s=2$ it beats only $78\%$ (ARC-Challenge), $74\%$ (MMLU), $57\%$ (GSM8K), and
$50\%$ (HellaSwag) of random draws on held-out coverage, and its mean coverage
sits inside the random $[p_{05},p_{95}]$ interval on all four datasets, for
instance $.182$ against a random spread of $[.146,.208]$ on HellaSwag.  A
five-fold design with one random draw per fold showed only the mean gap and
would have supported a stronger claim than the data warrant; the distribution
shows that choosing the complementarity-maximizing pool is often no better than
picking two alternatives at random.  Both facts are consistent with a submodular
objective whose coverage curve saturates early: greedy reliably finds a
near-optimal set, and there are many near-optimal sets to find.  Because the
cost axis is unbatched single-stream latency, the analysis reports relative
latency rather than a deployment Pareto frontier.





\section{Pilot Tables}
\label{app:pilottables}

These tables report the four-model pilot used in the pilot--confirmatory
comparison of Section~\ref{sec:confirm-opportunity}.

\begin{table*}[h]
\centering
\small
\begin{tabular}{lrrrrll}
\toprule
Dataset & $n$ & Empirical best & Oracle & Gap &
Naive 95\% interval & Selection-valid 95\% interval \\
\midrule
ARC-Challenge & 150 & .853 & .953 & .100 & [.053,.153] & [.000,.225] \\
HellaSwag & 150 & .553 & .813 & .260 & [.193,.333] & [.067,.441] \\
GSM8K & 100 & .430 & .580 & .150 & [.080,.220] & [.000,.394] \\
\bottomrule
\end{tabular}
\caption{Exploratory outcome-oracle opportunity.  The simultaneous interval
accounts for selecting the empirical best among four models on the same data.
Only HellaSwag has a positive lower limit at the pilot sample size.}
\label{tab:pilot-main}
\end{table*}

\begin{table}[h]
\centering
\small
\resizebox{\columnwidth}{!}{%
\begin{tabular}{lrrl}
\toprule
Dataset & Anchor & $\widehat{G}_b$ & Split-valid 95\% interval \\
\midrule
ARC-Challenge & Qwen-3B & .133 & [.066,.232] \\
HellaSwag & Qwen-3B & .267 & [.171,.381] \\
GSM8K & Qwen-1.5B & .140 & [.058,.267] \\
\bottomrule
\end{tabular}%
}
\caption{Split-selected opportunity (Corollary~\ref{cor:split}) at a
predeclared 50/50 split.  The anchor is selected on the selection half; the
exact Clopper--Pearson interval for $G_b$ uses the disjoint test half.  Every
lower limit is positive, and remains positive on all 20 reseeds.}
\label{tab:pilot-split}
\end{table}

\begin{table}[h]
\centering
\small
\resizebox{\columnwidth}{!}{%
\begin{tabular}{lrrrr}
\toprule
Dataset & TinyLlama & Qwen-.5B & Qwen-1.5B & Qwen-3B \\
\midrule
ARC-C & .240 & .293 & .700 & .853 \\
HellaSwag & .233 & .293 & .500 & .553 \\
GSM8K & .040 & .280 & .430 & .400 \\
\bottomrule
\end{tabular}%
}
\caption{Pilot per-model accuracy.  The GSM8K best fixed model is Qwen2.5-1.5B,
not Qwen2.5-3B.}
\label{tab:pilot-models}
\end{table}

\begin{table}[h]
\centering
\small
\resizebox{\columnwidth}{!}{%
\begin{tabular}{lrrr}
\toprule
Dataset & Rescue rate & Loss rate & Net gain \\
\midrule
ARC-Challenge & .033 & .027 & $+.007$ \\
HellaSwag & .047 & .020 & $+.027$ \\
GSM8K & .100 & .130 & $-.030$ \\
\bottomrule
\end{tabular}%
}
\caption{Rescue/loss decomposition of the max-confidence router on the pilot
outcomes, relative to the frozen best fixed model, recomputed from the committed
outcome matrices.  Net gain is rescue rate minus loss rate.  The negative GSM8K
net gain explains how a nonzero rescue rate produces no realized improvement.}
\label{tab:pilot-rescueloss}
\end{table}

\section{Proofs of the Main Results}
\label{app:bodyproofs}

This section proves the results stated in
Sections~\ref{sec:setup}--\ref{sec:submodular}.

\begin{proof}
For each outcome vector,
\[
Y-R_{b^\star}
=
\ind\!\left\{
R_{b^\star}=0,\ \exists m\neq b^\star:R_m=1
\right\}.
\]
Indeed, when $R_{b^\star}=1$ both $Y$ and $R_{b^\star}$ equal one.
When $R_{b^\star}=0$, their difference is one exactly when another model is
correct.  Taking expectations gives
$q-p^\star=G_{\mathrm{out}}$ and proves the identity.
\end{proof}

\begin{proof}
For every value of $Z$,
$\max_m\mu_m(Z)\geq\mu_{b^\star}(Z)$.  Taking expectations gives
$q_{\Z}\geq\E R_{b^\star}=p^\star$, proving the lower bound.  Equality of the
expectations of two ordered integrable random variables holds exactly when they
are equal almost surely, which gives the first equality condition.

For the upper bound, conditional monotonicity of expectation gives
\[
\max_m\E[R_m\mid Z]
\leq
\E[\max_m R_m\mid Z].
\]
Taking expectations yields $q_{\Z}\leq q$ and hence
$G_{\Z}\leq G_{\mathrm{out}}$.  The same ordered-random-variable argument gives
the second equality condition.
\end{proof}

\begin{proof}
The oracle mean is unbiased: $\E\widehat q=q$.  Since the maximum is convex,
\[
\E\max_m\widehat p_m
\geq
\max_m\E\widehat p_m
=p^\star.
\]
Subtracting gives $\E\widehat G\leq q-p^\star$.
\end{proof}

\begin{proof}
By the union bound, with probability at least
\[
1-\frac{\delta}{2}
-\sum_{m=1}^K\frac{\delta}{2K}
=1-\delta,
\]
all $K+1$ intervals cover simultaneously.  On this event,
\[
L_q\leq q\leq U_q
\quad\text{and}\quad
\max_m L_m\leq p^\star\leq\max_m U_m.
\]
Subtracting the largest possible value of $p^\star$ from the smallest possible
value of $q$ gives
$G_{\mathrm{out}}\geq L_q-\max_mU_m$.  Subtracting the smallest possible value
of $p^\star$ from the largest possible value of $q$ gives
$G_{\mathrm{out}}\leq U_q-\max_mL_m$.  Finally,
$G_{\mathrm{out}}\in[0,1]$, so truncation to $[0,1]$ preserves coverage.
\end{proof}

\begin{proof}
Conditional on the selection data, $b$ is fixed.  The pointwise identity in
Proposition~\ref{prop:complementarity} shows that each $\Delta_i$ is the
indicator of a rescue event relative to $b$.  Independence of the test examples
then yields an i.i.d.\ Bernoulli sample, to which the Clopper--Pearson guarantee
applies.
\end{proof}

\begin{proof}
The union bound gives simultaneous coverage of the oracle, all $K$ fixed
models, and all $L$ routers because
\[
\frac{\delta}{3}
+K\frac{\delta}{3K}
+L\frac{\delta}{3L}
=\delta.
\]
On this event,
\begin{align*}
\max_mL_m&\leq p^\star\leq\max_mU_m,\\
\max_\pi L_\pi&\leq q_\Pi\leq\max_\pi U_\pi,\\
L_q&\leq q\leq U_q.
\end{align*}
Subtracting the extreme endpoints yields the two intervals.  Constant policies
ensure $G_\Pi\geq0$, and the pointwise inequality
$R_{\pi(Z)}\leq\max_mR_m$ ensures $H_\Pi\geq0$; both quantities are at most one,
so truncation preserves coverage.
\end{proof}

\begin{proof}
For any cell with $\Prb(Z=z)>0$,
\begin{align*}
&\Prb(Z=z)\max_m\E[R_m\mid Z=z]\\
&\qquad=\max_m\Prb(Z=z,R_m=1)
=\max_m\theta_{zm}.
\end{align*}
Zero-probability cells contribute zero, so summing gives the identity.  Each
indicator $\ind\{Z_i=z,R_{im}=1\}$ is Bernoulli with mean $\theta_{zm}$.
The union bound makes all $JK$ intervals simultaneous with probability at
least $1-\delta$.  On that event, maximizing within each cell and summing
preserves the displayed lower and upper inequalities.
\end{proof}

\begin{proof}
Normalization is immediate.  If $S\subseteq T$, every outcome covered by $S$
is covered by $T$, so $F_b(S)\leq F_b(T)$.

For submodularity, take $S\subseteq T$ and $j\notin T$.  The marginal gain from
adding $j$ is
\begin{align*}
&F_b(S\cup\{j\})-F_b(S)\\
&=\Prb\!\left(
R_b=0,\ R_j=1,\ R_m=0\ \forall m\in S
\right).
\end{align*}
Replacing $S$ by the larger set $T$ adds constraints to this event, so
\[
F_b(S\cup\{j\})-F_b(S)
\geq
F_b(T\cup\{j\})-F_b(T).
\]
Thus $F_b$ has diminishing returns and is submodular.  The classical greedy
guarantee for normalized monotone submodular maximization under a cardinality
constraint gives the stated approximation.  Replacing probability by the
empirical average leaves the coverage and diminishing-returns arguments
unchanged.
\end{proof}

\section{Proof Details and Extensions}

\subsection{A Hoeffding alternative}

Theorem~\ref{thm:simultaneous-ci} uses exact binomial intervals.  A closed-form
alternative follows from Hoeffding's inequality.  With probability at least
$1-\delta$, simultaneously
\[
|\widehat q-q|
\leq
\sqrt{\frac{\log(4/\delta)}{2n}}
=:\epsilon_q
\]
and
\[
\max_m|\widehat p_m-p_m|
\leq
\sqrt{\frac{\log(4K/\delta)}{2n}}
=:\epsilon_p.
\]
Therefore
\[
|\widehat G-G_{\mathrm{out}}|
\leq \epsilon_q+\epsilon_p.
\]
To verify the last step, the maximum is one-Lipschitz under the sup norm:
\[
\left|\max_m\widehat p_m-\max_mp_m\right|
\leq\max_m|\widehat p_m-p_m|.
\]
This interval is usually wider than the exact construction but makes the rate
$O(\sqrt{\log(K/\delta)/n})$ explicit.

\subsection{Finite discrete signals}

Theorem~\ref{thm:finite-signal} makes a coarse task taxonomy identifiable
without modeling assumptions.  To convert its interval for $q_\Z$ into one for
$G_\Z=q_\Z-p^\star$, allocate part of the total failure budget to simultaneous
fixed-model intervals and subtract their extreme endpoints exactly as in
Theorem~\ref{thm:simultaneous-ci}.  Continuous embeddings do not enjoy this
finite-cell identification without smoothness, complexity, or model-class
assumptions; a frozen clustering creates a certificate only for the cluster
identifier, not for all information in the embedding.

\subsection{Weighted pool budgets}

When candidate $m$ has acquisition cost $w_m>0$, the calibration objective
remains monotone submodular under
$\sum_{m\in S}w_m\leq B$.  Standard partial-enumeration or
cost-benefit greedy algorithms give constant-factor guarantees.  We omit a new
proof because the result is the classical monotone submodular knapsack problem;
the confirmatory study reports both cardinality and measured-cost constraints.

\section{Prompt Templates}
\label{app:prompts}

\subsection{ARC-Challenge}

\begin{quote}
Choose the correct answer. Return only one option label: A, B, C, or D.

Question: \emph{question text}

A. \emph{choice A}\\
B. \emph{choice B}\\
C. \emph{choice C}\\
D. \emph{choice D}

Answer:
\end{quote}

\subsection{HellaSwag}

\begin{quote}
Choose the most plausible continuation. Return only one option label:
A, B, C, or D.

Context: \emph{context text}

A. \emph{ending A}\\
B. \emph{ending B}\\
C. \emph{ending C}\\
D. \emph{ending D}

Answer:
\end{quote}

\subsection{GSM8K}

\begin{quote}
Solve the problem. Show brief reasoning, then end with the final numeric answer
in the exact format \texttt{\#\#\#\# <number>}.

Problem: \emph{problem text}

Solution:
\end{quote}

\subsection{MMLU}

\begin{quote}
Choose the correct answer. Return only one option label: A, B, C, or D.

Subject: \emph{subject name}

Question: \emph{question text}

A. \emph{choice A}\\
B. \emph{choice B}\\
C. \emph{choice C}\\
D. \emph{choice D}

Answer:
\end{quote}

\subsection{Code generation}

\begin{quote}
Complete the function described below.  Return only executable code, without
Markdown fences or additional explanation.

Specification: \emph{function signature and problem statement}

Code:
\end{quote}

\subsection{Factual question answering}

\begin{quote}
Answer the question using the shortest factual answer span that is sufficient.
Return only the answer.

Question: \emph{question text}

Answer:
\end{quote}

\section{Paired Router Contrasts}
\label{app:paired}

\begin{table*}[h]
\centering
\small
\begin{tabular}{llrrrrrr}
\toprule
Dataset & Method & Accuracy & 95\% CI & $n_{10}$ & $n_{01}$ &
Holm $p$ & Recovered \% \\
\midrule
ARC-C & Max confidence & .860 & [.816,.897] & 14 & 18 & 1.000 & $-13.8$ \\
ARC-C & Plurality vote & .850 & [.804,.888] & 7 & 14 & 1.000 & $-24.1$ \\
HellaSwag & Max confidence & .640 & [.583,.694] & 32 & 37 & 1.000 & $-6.0$ \\
HellaSwag & Plurality vote & .620 & [.562,.675] & 32 & 43 & 1.000 & $-13.1$ \\
GSM8K & Max confidence & .545 & [.473,.615] & 16 & 39 & \textbf{.021} & $-57.5$ \\
GSM8K & Plurality vote & .635 & [.564,.702] & 14 & 19 & 1.000 & $-12.5$ \\
MMLU & Max confidence & .580 & [.522,.636] & 27 & 35 & 1.000 & $-8.7$ \\
MMLU & Plurality vote & .540 & [.482,.597] & 22 & 42 & .118 & $-21.7$ \\
\bottomrule
\end{tabular}
\caption{Per-router paired statistics against the frozen best fixed model for
every dataset and selector.  Accuracy intervals
are exact Clopper--Pearson; $n_{10}$ and $n_{01}$ are the discordant counts
(router right/baseline wrong and the reverse); $p$ is an exact McNemar test
Holm-corrected across the eight rows.  \emph{Every} recovered fraction is
negative: routing costs accuracy on all four datasets.  Only one row survives
Holm correction, and it is a harm rather than a gain: max-confidence on GSM8K
loses $39$ baseline successes while rescuing $16$ ($p=.021$, recovered
$-57.5\%$).  The remaining rows are indistinguishable from the baseline, so the
honest reading is no realized benefit and one significant regression.}
\label{tab:paired-template}
\end{table*}

\begin{table*}[h]
\centering
\small
\begin{tabular}{llrrrrrr}
\toprule
Dataset & Model & AUROC & AUPRC & Brier & ECE & Rescue hit & Loss rate \\
\midrule
ARC-C & TinyLlama & .558 & .285 & .258 & .273 & .040 & .673 \\
ARC-C & Qwen-0.5B & .626 & .479 & .339 & .370 & .017 & .593 \\
ARC-C & Qwen-1.5B & .782 & .913 & .237 & .236 & .023 & .170 \\
ARC-C & Qwen-3B & .714 & .902 & .153 & .154 & .050 & .090 \\
ARC-C & Llama-3.2-3B & .771 & .898 & .165 & .114 & .030 & .167 \\
ARC-C & Gemma-2-2B & .814 & .905 & .179 & .148 & .017 & .187 \\
ARC-C & Phi-3.5-mini$^{\dagger}$ & .843 & .949 & .107 & .109 & .000 & .000 \\
ARC-C & Mistral-7B & .686 & .818 & .231 & .231 & .047 & .173 \\
\midrule
MMLU & TinyLlama & .531 & .235 & .255 & .299 & .113 & .513 \\
MMLU & Qwen-0.5B & .676 & .571 & .271 & .283 & .070 & .363 \\
MMLU & Qwen-1.5B & .718 & .729 & .356 & .352 & .103 & .207 \\
MMLU & Qwen-3B & .665 & .723 & .384 & .391 & .123 & .153 \\
MMLU & Llama-3.2-3B & .719 & .745 & .285 & .252 & .110 & .223 \\
MMLU & Gemma-2-2B & .722 & .731 & .294 & .268 & .103 & .253 \\
MMLU & Phi-3.5-mini$^{\dagger}$ & .735 & .769 & .315 & .309 & .000 & .000 \\
MMLU & Mistral-7B & .722 & .617 & .448 & .462 & .077 & .237 \\
\bottomrule
\end{tabular}
\caption{Confidence and rescue-set audit on the two four-option datasets, all
eight models.  AUROC and AUPRC measure how well a model's own confidence ranks
its correctness; AUPRC must be read against that model's accuracy as prevalence,
which is why TinyLlama's $.285$ on ARC-Challenge is near its $.240$ base rate
while Phi-3.5-mini's $.949$ sits above its $.873$.  ``Rescue hit'' is the
fraction of items the model answers correctly while the frozen baseline fails,
and ``loss rate'' the reverse.  $^{\dagger}$Phi-3.5-mini \emph{is} the baseline on
both datasets, so its rescue and loss rates are zero by construction.  The
diagnostic explains the negative results in Table~\ref{tab:paired-template}: the
models with the most to contribute are the weakest ones, and their confidence is
also the worst calibrated (TinyLlama ECE $.273$/$.299$), so a confidence-gated
router cannot identify the items they would rescue.}
\label{tab:confidence-template}
\end{table*}

\section{Pilot Integer Counts}
\label{app:pilotcounts}

The selection-valid intervals in Table~\ref{tab:pilot-main} use the following
counts.  ARC-Challenge has per-model correct counts
$(36,44,105,128)$ and oracle count $143$ out of 150.  HellaSwag has
$(35,44,75,83)$ and oracle count $122$ out of 150.  GSM8K has
$(4,28,43,40)$ and oracle count $58$ out of 100.  The intervals allocate
$\delta/2$ to the oracle binomial interval and $\delta/(2K)$ to each fixed
model interval.

\section{Pairwise Routing Appendix}
\label{app:pairwise}

RouteLLM is not a direct baseline for our primary $K$-way correctness-routing
setting.  Its released formulation selects between a designated strong model and
weak model using pairwise preference supervision, whereas our setting permits
multiple correct models and contains no uniformly dominant endpoint.  We therefore
evaluate RouteLLM only in a supplementary two-model experiment with a predeclared
strong--weak pair.  We do not treat a newly constructed pairwise composition as the
published RouteLLM method.

HybridLLM \citep{ding2024hybridllm} is handled identically.  Platform-controlled
routers such as OpenRouter \citep{openrouter2025auto} define a changing external
pool and therefore fall outside the frozen-pool estimand studied here.



\end{document}